\def\year{2021}\relax
\documentclass[letterpaper]{article} 
\usepackage{aaai21}  
\usepackage{times}  
\usepackage{helvet} 
\usepackage{courier}  
\usepackage[hyphens]{url}  
\usepackage{graphicx} 
\usepackage{natbib}  
\usepackage{caption} 
\usepackage{booktabs}
\usepackage{amsmath,amssymb,amsfonts,amsthm}
\usepackage{algorithm}
\usepackage{algpseudocode}
\usepackage{dsfont}
\usepackage{graphicx}
\usepackage{subfigure}
\usepackage{enumitem}
\usepackage[font=small]{caption}

\newtheorem{theorem}{Theorem}
\newtheorem{corollary}{Corollary}
\newtheorem{lemma}{Lemma}
\algnewcommand{\LeftComment}[1]{ \(\triangleright\) #1}
\usepackage{multirow}
\usepackage[switch]{lineno}
\usepackage{txfonts} 
\ifodd 0
\newcommand{\shir}[1]{{\color{red}#1}}
\else
\newcommand{\shir}[1]{#1}
\fi

\ifodd 0
\newcommand{\shic}[1]{{\color{red}(Chengshuai: #1)}}
\else
\newcommand{\shic}[1]{}
\fi

\ifodd 0
\newcommand{\congr}[1]{{\color{blue}#1}}
\else
\newcommand{\congr}[1]{#1}
\fi

\ifodd 0
\newcommand{\congc}[1]{{\color{magenta}(Cong: #1)}}
\else
\newcommand{\congc}[1]{}
\fi

\title{Federated Multi-Armed Bandits}
\author{
    Chengshuai Shi,
    Cong Shen
    \\
}
\affiliations{

    Department of Electrical and Computer Engineering\\
    University of Virginia\\
    Charlottesville, VA 22904 \\
    \{cs7ync, cong\}@virginia.edu \\
}

\begin{document}
\maketitle

\begin{abstract}
Federated multi-armed bandits (FMAB) is a new bandit paradigm that parallels the federated learning (FL) framework in supervised learning. It is inspired by practical applications in cognitive radio and recommender systems, and enjoys features that are analogous to FL. This paper proposes a general framework of FMAB and then studies two specific federated bandit models. We first study the approximate model where the heterogeneous local models are random realizations of the global model from an unknown distribution. This model introduces a new uncertainty of \textit{client sampling}, as the global model may not be reliably learned even if the finite local models are perfectly known. Furthermore, this uncertainty cannot be quantified \textit{a priori} without knowledge of the suboptimality gap. We solve the approximate model by proposing Federated Double UCB (Fed2-UCB), which constructs a novel ``double UCB'' principle accounting for uncertainties from both arm and client sampling. We show that gradually admitting new clients is critical in achieving an $O(\log(T))$ regret while explicitly considering the communication \congr{cost}. The exact model, where the global bandit model is the exact average of heterogeneous local models, is then studied as a special case. We show that, somewhat surprisingly, the order-optimal regret can be achieved independent of the number of clients with a careful choice of the update periodicity.  Experiments using both synthetic and real-world datasets corroborate the theoretical analysis and provide interesting insight into the proposed algorithms.
\end{abstract}

\section{Introduction}
\label{sec:intro}

Federated learning (FL) \cite{mcmahan2017communication} is a new distributed machine learning (ML) paradigm that addresses new challenges in modern machine learning (ML). In particular, FL handles distributed ML with the following characteristics: 

\begin{itemize}[leftmargin=*,itemindent=\dimexpr\labelsep+\labelwidth \relax]\itemsep=0pt
    \item \textbf{Non-IID local datasets.}
    FL caters to the growing trend that massive amount of the real-world data are generated directly at the edge devices. The local datasets are likely drawn from non-independent and identically distributed (non-IID) distributions, and do not represent the global distribution. 
    
    \item \textbf{Massively distributed.}
    The number of participating clients can be significant, e.g., on the order of millions \cite{bonawitz2019towards}.
    
    \item \textbf{Communication efficiency.} 
    The communication cost scales with the number of clients, which becomes one of the primary bottlenecks of the FL system \cite{mcmahan2017communication}. It is critical to minimize the communication cost while maintaining the learning accuracy.
    
    \item  \textbf{Privacy.} 
    FL protects the local data privacy by only sharing model updates instead of the raw data.
\end{itemize}

While the state of the art FL largely focuses on the supervised learning setting, we propose to extend the core principles of FL to the multi-armed bandits (MAB) problem. This is motivated by real-world applications, such as:
\begin{itemize}[leftmargin=*,itemindent=\dimexpr\labelsep+\labelwidth \relax]\itemsep=0pt
    
    \item  \textbf{Cognitive radio.}
    A base station wants to select one channel from a given set of channels that is most likely to be ``empty'' in its  coverage area. It is well-known that different geographic locations have different channel availabilities, and the (ground-truth) global channel availability is the average over the entire coverage area (see Section~\ref{sec:apprmod} for a detailed discussion). The base station, however, is fixed at one location and cannot learn the global channel availability by itself. A common solution is to utilize randomly placed devices (e.g., mobile phones) in the coverage area to sample the channels and then aggregate at the base station. Each device is at a different location and thus samples a non-IID local channel availability.  In addition, the bandit problem is approximate because there are only {finite} devices while the global model is integrated over the entire coverage area.
    
    \item \textbf{Recommender system.}
    The central server wants to recommend the most popular item to new customers to maximize the expected reward. The server does not initially have the global item popularity but can learn via interacting with customers, leading to a bandit problem \cite{li2010contextual}. In reality, however, the server may not learn the popularity model \textit{directly} from user behavior data due to privacy concerns or regulation requirement (e.g., private user data in some regions may not be shared outside). Instead, the user data \congr{are} stored strictly on the client device or local server (and never leave) for better privacy preservation. The local view of item popularity is often biased and not representative of the overall distribution, while the global server can only access some {aggregate} information instead of individual data, resulting in a federated learning problem in the bandit setting.
\end{itemize} 

\begin{figure*}[thb]
\centering
	\subfigure[Cognitive radio]{ \includegraphics[width=0.4\linewidth]{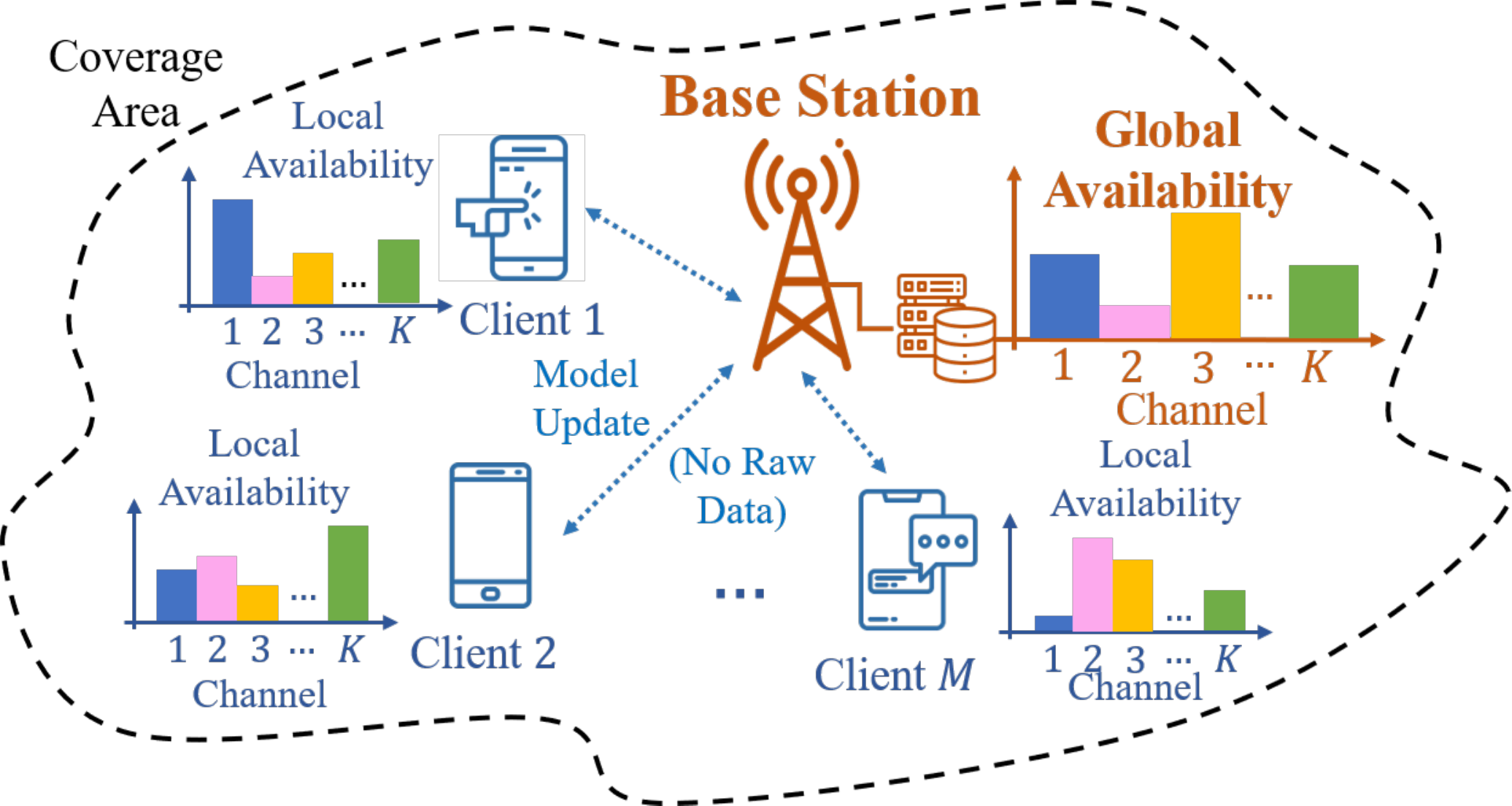}}
	\subfigure[Recommender system]{ \includegraphics[width=0.4\linewidth]{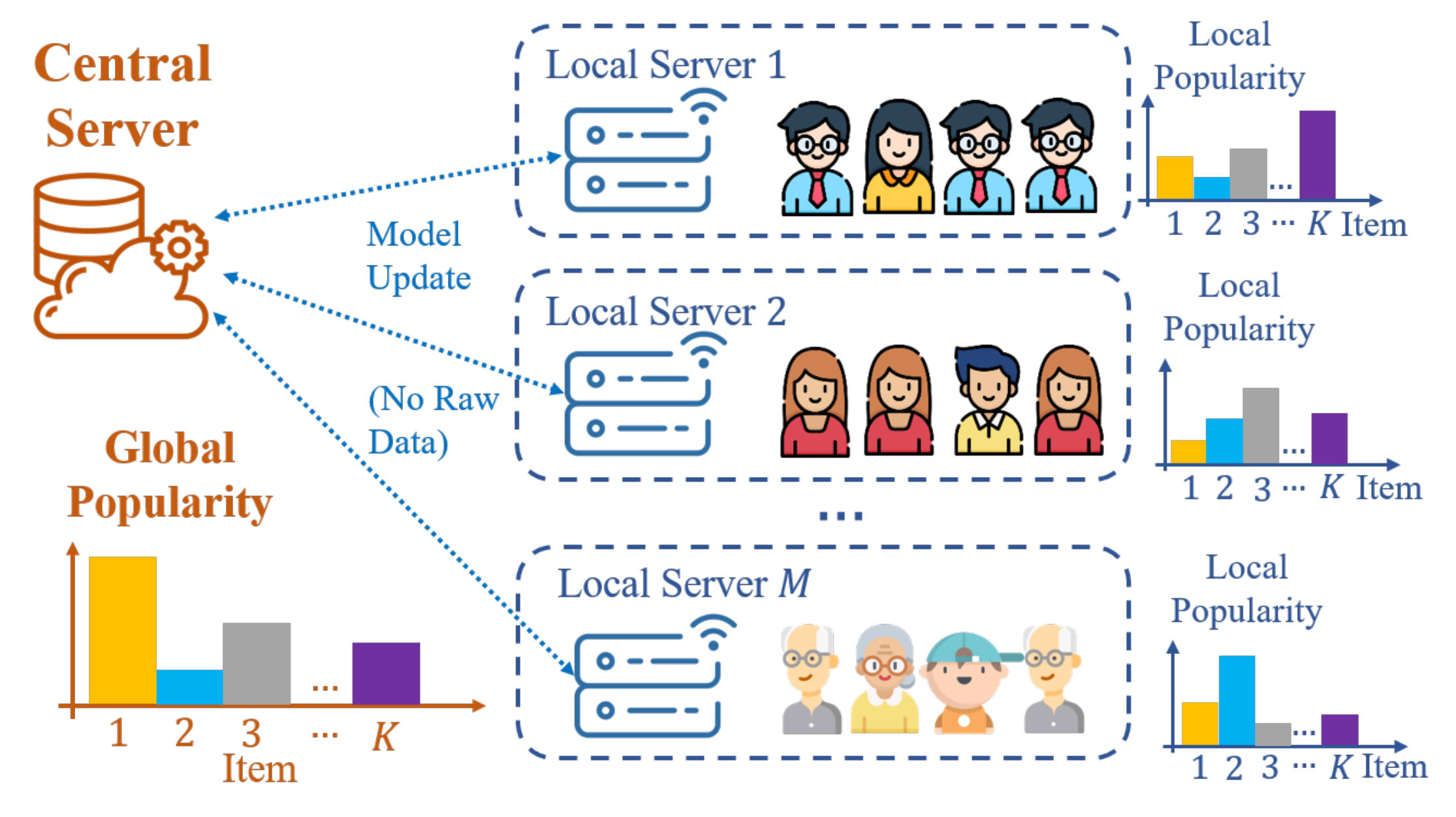}}
	\caption{Motivating examples for federated multi-armed bandits.}
	\label{fig:example}
\end{figure*}

In both applications, which are illustrated in Fig.~\ref{fig:example}, we see that the general FL characteristics still apply, and yet the underlying problem is a bandit one. This leads to a natural marriage of FL and MAB, and we are motivated to solve a global stochastic MAB problem from (possibly a large number of) local bandit models that are non-IID, in a communication-efficient and privacy-preserving manner. 

In this work, a novel framework of Federated MAB (FMAB) is developed, which represents the first {\em systematic} attempt to bridge FL and MAB to the best of our knowledge. The FMAB framework is general and can incorporate a variety of bandit problems that share the FL principles. We demonstrate the merit of this framework by first studying an approximate FMAB model, where the global bandit model exists as a ground truth while local bandit models are random realizations of it. In addition to the usual reward uncertainty from arm sampling, this setting introduces a new uncertainty associated with client sampling. In particular, the approximate model does not assume any suboptimality gap knowledge, which prohibits determining the requirement for client sampling \textit{a priori}. Mixing client sampling with arm sampling without the knowledge of suboptimality gap significantly complicates the problem, and we address these challenges by proposing a novel {Federated Double UCB (Fed2-UCB)} algorithm that gradually samples new clients while performing arm sampling, and thus simultaneously explores and balances both types of uncertainty. Theoretical analysis shows that Fed2-UCB achieves an $O(\log(T))$ regret (which explicitly considers communication cost) that approaches the lower bound of the standard stochastic MAB model with an additional term of communication loss. As a special case, the exact FMAB model is then studied, where the global model is the exact average of all local models. The Fed1-UCB algorithm degenerates from Fed2-UCB and achieves an order-optimal regret upper bound which, somewhat surprisingly, is independent of the number of clients with a proper choice of the update periodicity. Numerical simulations on synthetic and real-world datasets demonstrate the effectiveness and efficiency of the proposed algorithms and offer some interesting insights.

\section{Problem Formulation}
\label{sec:model}
In the standard stochastic MAB setting, a single player directly plays $K$ arms, with rewards $X_k$ of arm $k\in[K]$ sampled independently from a $\sigma$-subgaussian distribution with mean $\mu_{k}$. At time $t$, the player chooses an arm $\pi(t)$ and the goal is to receive the highest expected cumulative reward in $T$ rounds, which is characterized by minimizing the (pseudo-)regret:
\begin{equation}
\label{eqn:regret_single}
	R(T)=\mathbb{E}\left[\sum\nolimits_{t=1}^T X_{k_*}(t)-\sum\nolimits_{t=1}^{T}X_{\pi(t)}(t)\right],
\end{equation}
where $k_*$ is the optimal arm with mean reward $\mu_*\doteq \mu_{k_*}=\max_{k\in[K]}\mu_k$, and the expectation is taken over \shir{the randomness of both policy and environment}. As shown by \citet{Lai:1985}, there exists a lower bound for the regret as:
\begin{equation}\label{eqn:single_lower}
	\liminf_{T\to\infty}\frac{R(T)}{\log(T)}\geq \sum\nolimits_{k\not=k_*}\frac{\mu_*-\mu_k}{\text{kl}(\mu_k,\mu_*)},
\end{equation}
where $\text{kl}(\mu_k,\mu_*)$ denotes the KL-divergence between the two corresponding distributions. 

In this section, we present a framework of FMAB as illustrated in Fig.~\ref{fig:framework}. The key aspects of the FL principles mentioned in Section \ref{sec:intro} are also elaborated, which become more clear when algorithm designs are presented. 

\begin{figure}[htb]
	\centering
	\includegraphics[width=0.4\textwidth]{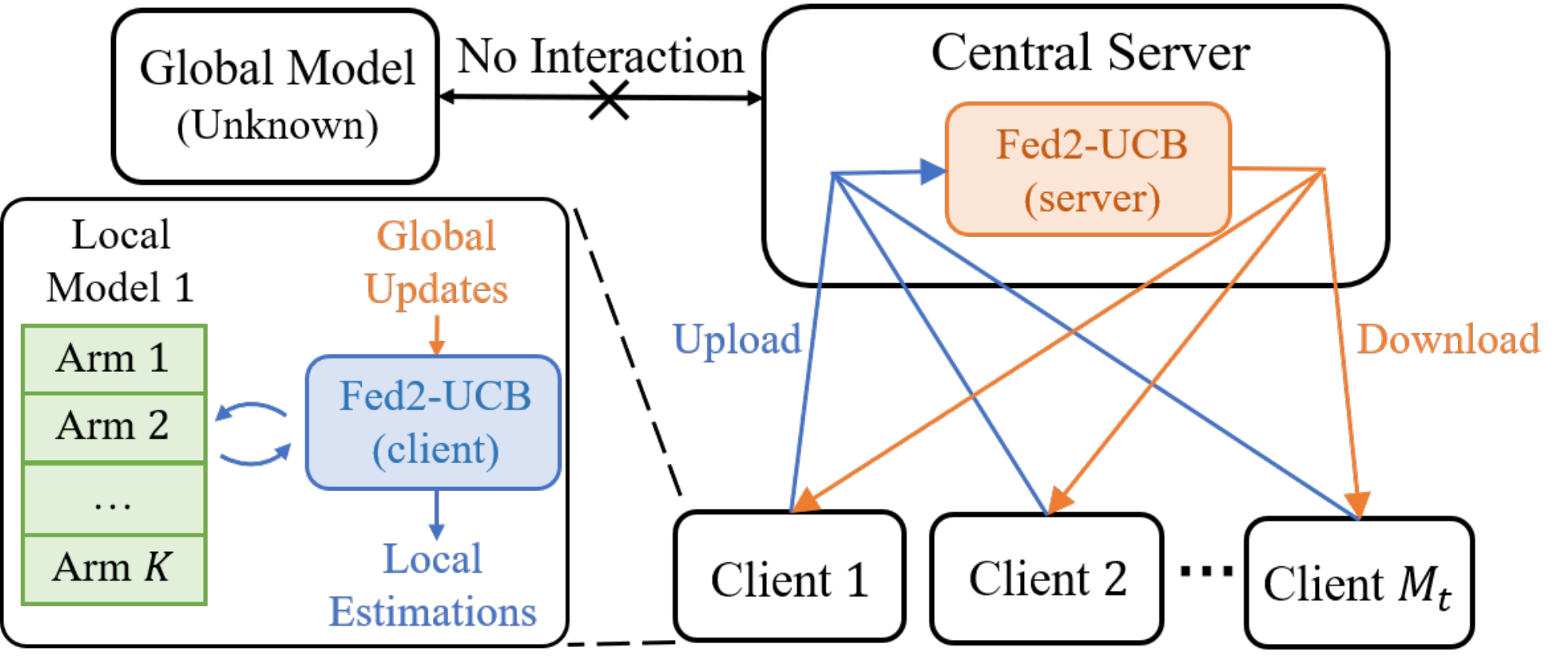}
	\caption{The FMAB framework.}
	\label{fig:framework}
\end{figure}

\textbf{Clients.} Multiple clients interact with the same set of $K$ arms (referred as ``local arms'') in the FMAB framework. We denote $M_t$ as the number of participating clients at time $t$, who are labeled from $1$ to $M_t$ to facilitate discussions (they are not used in the algorithms). A client can only interact with her own local MAB model, and there is no direct communication between clients. Arm $k$ generates independent \textit{observations} $X_{k,m}$ for client $m$ following a $\sigma$-subgaussian distribution with mean $\mu_{k,m}$.  Note that $X_{k,m}$ is only an observation but not a reward. For different clients $n \neq m$, their models are non-IID; hence $\mu_{k,n}\neq \mu_{k,m}$ in general.

\textbf{Server.} There exists a central server with a global stochastic MAB model, which has the same set of $K$ arms (referred as ``global arms'') of $\sigma$-subgaussian reward distributions with mean reward $\mu_k$ for arm $k$. The true rewards for this system are generated on this global model, thus the learning objective is on the global arms. However, the server cannot directly observe rewards on the global model; she can only interact with clients who feed back information of their local observations. We consider the general non-IID situation where the local models are not \congr{necessarily} the same as the global model, and also make the common assumption that clients and the server are fully synchronized \cite{mcmahan2017communication,bonawitz2019towards}.

\textbf{Communication cost.} Although clients cannot communicate with each other, after certain time, they can transmit local ``model updates'' based on their local observations to the server, which aggregates these updates to have a more accurate estimation of the global model. The new estimation is then sent back to the clients to replace the previous estimation for future actions. However, just like in FL, the communication resource is a major bottleneck and the algorithm has to be conscious about its usage. We incorporate this constraint in FMAB by imposing a loss $C$ every time a client communicates to the server, which will be accounted for in the performance measure defined below.

\subsection{The Approximate Model}
\label{sec:apprmod}
Although the non-IID property of local models is an important feature of FMAB, there must exist some relationship between local and global models so that {observations} on local bandit models help the server learn the global model. Here, we propose the approximate FMAB model, where the global model is a fixed (but hidden) ground truth (i.e., exogenously generated regardless of the participating clients), and the local models are IID random realizations of it. 

Specifically, the global arm $k$ has a fixed mean reward of $\mu_k$. For client $m$, the mean reward $\mu_{k,m}$ of her local arm $k$ is a sample from an unknown distribution $\phi_k$, which is a $\sigma_c$-subgaussian distribution with mean $\mu_{k}$. For a different client $n\not=m$, $\mu_{k,n}$ is sampled IID from $\phi_k$. 
Since local models are stochastic realizations of the global model, a \textit{finite} collection of the former may not necessarily represent the latter.  In other words, \shir{if there are $M$ involving clients, although $\forall m\in[M], \mathbb{E}[\mu_{k,m}]=\mu_k$, the averaged local model $\hat{\mu}_{k}^{M}\doteq \frac{1}{M}\sum_{m=1}^{M}\mu_{k,m}$ may not be consistent with the global model. Specifically,  $\hat{\mu}_{k}^{M}$ is not necessarily equal (or even close) to $\mu_k$, which introduces significant difficulties.} Intuitively, the server needs to sample sufficiently many clients to have a statistically accurate estimation of the global model, but as we show later, the required number of clients cannot be obtained \textit{a priori} without the suboptimality gap knowledge. The need of client sampling also coincides with the property of massively distributed clients in FL.

\textbf{Motivation Example.} The approximate model \congr{captures the key characteristics of} a practical cognitive radio system, as illustrated in Fig.~\ref{fig:example}(a). Assume a total of $K$ candidate channels, indexed  by $\{1,...,K\}$. Each channel's availability is location-dependent, with $p_k(x)$ denoting the probability that channel $k$ is available at location $x$.  
The goal of the base station is to choose one channel out of $K$ candidates to serve all potential cellular users (e.g., control channel) in the given coverage area $\mathcal{D}$ with area $D$. Assuming users are uniformly randomly distributed over $\mathcal{D}$, the global channel availability is measured throughout the entire coverage area as
\begin{equation}
    p_k = \mathbb{E}_{x\sim u(\mathcal{D})}\left[p_k(x)\right]=\oiint_{\mathcal{D}}\frac{1}{D}p_k(x)dx.
    \label{eqn:crcontinous}
\end{equation}
It is well known in wireless research that a base station cannot directly sample $p_k$ by itself, because it is fixed at one location\footnote{The best it can do by itself is to estimate $p_k(x_{\text{BS}})$ where $x_{\text{BS}}$ is the location of the base station, possibly through the Network Listen Mode \cite{3gpp:36921}.}. In addition, Eqn.~\eqref{eqn:crcontinous} requires a \textit{continuous} sampling throughout the coverage area, which is not possible in practice. Realistically, the base station can only direct cellular user $m$ at \textit{discrete} location $x_m$ to estimate $p_k(x_m)$, and then aggregate observations from finite number of users as \shir{$\hat{p}_{k}= \frac{1}{M}\sum_{m=1}^{M}p_k(x_m)$} to approximate $p_k$.  Clearly, even if $p_k(x_m)$ are perfect, $\hat{p}_{k}$  may not necessarily represent $p_k$ well.

\textbf{Regret definition.} Without loss of generality, we assume there is only one optimal global arm $k_*$ with $\mu_*\doteq\mu_{k_*} = \max_{k\in[K]}\mu_k$, and $\Delta=\mu_*-\max_{k\neq k_*}\{\mu_k\}$ denotes the suboptimality gap of the global model (both unknown to the algorithm). We further denote $\gamma_1,\cdots,\gamma_{T_c}$ as the time slots when the clients communicate with the central server for both upload and download. The notion of (pseudo-)regret in Eqn.~\eqref{eqn:regret_single} for the single-player model can be generalized to all the clients with additional communication loss, as follows:
\begin{equation}\label{eqn:regret_fed}
R(T)=\mathbb{E}\Bigg[\underbrace{\sum_{t=1}^{T}M_tX_{k_*}(t)-\sum_{t=1}^T\sum_{m=1}^{M_t}X_{\pi_m(t)}(t)}_{\text{exploration and exploitation}}+\underbrace{\sum_{\tau=1}^{T_c}CM_{\gamma_\tau}}_{\text{communication}}\Bigg],
\end{equation}
where $\pi_m(t)$ is the arm chosen by client $m$ at time $t$. In this work, we aim at designing algorithms with $O(\log(T))$ regret as in the single-player setting.

Several comments are in place for Eqn. \eqref{eqn:regret_fed}. First, the reward oracle is defined with respect to the {single} global optimal arm but not the distinct local optimal arms. This choice is analogous to that the reward oracle of the pseudo-regret in the single-player MAB model is defined with respect to \textit{one} optimal arm throughout the horizon but not the arms with the highest reward at every time slot \cite{Bubeck:2012}. Second, the cumulative reward of the system is defined on the global model, because clients only receive \textit{observations} from playing the local bandit game, and the \textit{reward} is generated at the system-level global model. Taking the cognitive radio system as an example, the choice by each client only produces her observation of the channel availability, but the reward is generated by the base station when this channel is used for the entire coverage area. Lastly, regret definition in Eqn. \eqref{eqn:regret_fed} discourages the algorithm to involve too many clients. Ideally, only sufficiently many clients should be admitted to accurately reconstruct the global model, and any more clients would result in more communication loss without improving the model learning.

\section{Fed2-UCB for Approximate FMAB}

\subsection{Challenges and Main Ideas}

The first and foremost challenge in the approximate model \shir{comes from that the local models are only stochastic realizations of the global model. Even with the perfect information of all local arms, the optimal global arm may not be produced faithfully.} We refer to this new problem as the \textit{uncertainty from client sampling}. How to simultaneously handle the two types of uncertainty (client sampling and arm sampling) is at the center of solving the approximate model. 

A second issue comes from the conflict between non-IID local models and the global model. In particular, the globally optimal arm may be sub-optimal for a client's local model, and hence it cannot be correctly inferred by the client individually. Communication between clients and the server is key to address this conflict, but the challenge is how to control the communication loss and balance the overall regret.

In this section, we first characterize the uncertainty from client sampling
by analyzing the probability that the averaged local model does not faithfully represent the global model, and illustrate that without knowledge of the suboptimality gap $\Delta$, the algorithm cannot determine \textit{a priori} the number of required clients. Then, Federated Double UCB (Fed2-UCB) is proposed, in which a novel ``double UCB'' principle carefully balances and trades off the two sources of uncertainty while controlling the communication cost.

\subsection{Client Sampling}
In the approximate model, the key to determine whether the local knowledge is sufficient lies in whether the optimal global arm can be inferred correctly. \shir{When there are $M$ involving clients}, the best approximate of the global model is the averaged local model, i.e., \shir{$\hat{\mu}_k^{M}$}. Although the utilities of local arms may be different from the global model, \shir{if the true optimal global arm is still optimal in this averaged local model, i.e., $\hat{\mu}_{k_*}^{M}>\max_{k\neq k_*}\hat{\mu}_{k}^{M}$}, a sub-linear regret can be achieved with local knowledge. Otherwise, \shir{arm $k_*$ is not optimal with respect to $\hat{\mu}_k^{M}$, and} no matter how many explorations are performed locally (even with perfect local knowledge), the global optimal arm cannot be found \shir{using the sampled $M$ local models} and thus a linear regret occurs.

The following theorem characterizes the accuracy of representing the global model by the averaged local model from a fixed number of clients.
\begin{theorem}\label{thm:accuracy}
	 With $M$ involved clients, denote $P_z=\mathbb{P}\left(\hat{\mu}_{k_*}^M\leq\max_{k\in[K]}\hat{\mu}_k^M\right)$, the following result holds:
    \begin{equation*}
    \small
        \begin{aligned}
        P_z &= O\left(\sum\nolimits_{k\not=k_*}\exp\left\{-\sigma_c^{-2}M(\mu_*-\mu_k)^2\right\}\right)
        = O\left(K\exp\left\{-\sigma_c^{-2}M\Delta^2\right\}\right).
        \end{aligned}
    \end{equation*} 
\end{theorem}
Theorem \ref{thm:accuracy} indicates that the probability that the averaged local model does not represent the global model, \shir{i.e., $\hat{\mu}_{k_*}^M\leq \max_{k\in[K]}\hat{\mu}_k^M$,} decreases exponentially with respect to the number of involved clients $M$. Thus, it is fundamental to involve a sufficiently large number of clients in order to reconstruct the global model correctly. More specifically, to guarantee that $P_z= O(1/T)$, by which the overall regret can scale sub-linearly, it is sufficient to sample $M$ clients with
\begin{equation}\label{eqn:m_accuracy}
    M= \Omega \left(\sigma_c^2\Delta^{-2}\log(KT)\right).
\end{equation}
If Eqn.~\eqref{eqn:m_accuracy} is satisfied throughout the bandit game, the optimal arm can be successfully found. However, clients do not have access to the knowledge of $\Delta$. Thus, the requirement in Eqn. \eqref{eqn:m_accuracy} cannot be guaranteed in advance.

\congr{On the other hand, involving too many clients may be detrimental to the regret, as can be seen in Eqn. \eqref{eqn:regret_fed}. Specifically, in order to have an $O(\log(T))$ regret, $M$ should satisfy:}
\begin{equation}\label{eqn:m_upper}
    M = O\left(\log(T)\right).
\end{equation}
Comparing Eqns. \eqref{eqn:m_accuracy} and \eqref{eqn:m_upper} suggests that $M$ has to be $\Theta(\log(T))$ to achieve a correct representation of the global model while maintaining an $O(\log(T))$ regret.

\subsection{The Fed2-UCB Algorithm}
With the unknown requirement in Eqn. \eqref{eqn:m_accuracy}, it is \congr{unwise} to only admit a small number of clients in the whole game. On the other hand, Eqn. \eqref{eqn:m_upper} prohibits involving too many clients \congr{to achieve an $O(\log(T))$ regret}. There are also practical system considerations that prevent having too many clients, which has been discussed in the context of FL \cite{bonawitz2019towards}. We propose the Fed2-UCB algorithm where the central server gradually admits new clients into the game after each communication round while keeping local clients gathering observations. The method of gradually increasing the clients ensures that the server samples a set of small but sufficiently representative clients based on the underlying statistical structure of the bandit game. The proposed ``double UCB'' principle simultaneously addresses the uncertainty from both client sampling and arm sampling. 

\begin{algorithm}[htb]
        \small
		\caption{Fed2-UCB: client $m$}
		\label{alg:fed2_local}
		\begin{algorithmic}[1]
			\State Initialize $p\gets 1$; $[K_1]\gets [K]$
			\While{$K_{p}>1$} 
			\State \shir{Pull each active arm $k\in [K_p]$ for $f(p)$ times}
			\State \shir{Calculate the local sample means $\bar{\mu}_{k,m}(p),\forall k\in[K_p]$}
			\State Send local updates $\bar{\mu}_{k,m}(p), \forall k\in[K_p]$ to the server
			\State Receive global update set $E_p$ from the server
			\State $[K_{p+1}]\gets[K_p]\backslash E_p$; $p\gets p+1$
			\EndWhile
			\State $F\gets$the only element in $[K_p]$; Stay on arm $F$ until $T$
		\end{algorithmic}
\end{algorithm}
\begin{algorithm}[htb]
        \small
		\caption{Fed2-UCB: central server}
		\label{alg:fed2_central}
		\begin{algorithmic}[1]
			\State Initialize $p\gets 1$; $[K_1]\gets [K]$		\While{$K_{p}>1$}
            \State Admit $g(p)$ new clients  \Comment{\textit{Client sampling}}
			\State Receive local updates $\bar{\mu}_{k,m}(p), \forall k\in[K_p],\forall m\in[M(p)]$
			\State Calculate $\forall k\in[K_p],\bar{\mu}_{k}(p)\gets \sum_{m=1}^{M(p)}\bar{\mu}_{k,m}(p)/M(p)$
			\State $E_p\gets \left\{k\in[K_p]|\bar{\mu}_k(p)+B_{p,2}\leq \max\nolimits_{l\in[K_p]}\bar{\mu}_l(p)-B_{p,2}\right\}$ 
			\State Send global update set $E_p$ to all involved clients
			\State $[K_{p+1}]\gets[K_p]\backslash E_p$; $p\gets p+1$
			\EndWhile
		\end{algorithmic}
\end{algorithm}

The Fed2-UCB algorithm is performed in phases simultaneously and synchronously at clients and the central server. Clients collect observations and update local estimations for the arms that have not been declared as sub-optimal, i.e., the active arms, while the server admits new clients and aggregates the local estimations as global estimations to eliminate sub-optimal active arms. We denote the set of active arms in the $p$-th phase by $[K_p]$ with cardinality $K_p$. The detailed algorithm for the clients and the central server are given in Algorithms \ref{alg:fed2_local} and \ref{alg:fed2_central}, respectively. 

At phase $p$, $g(p)$ new clients are first added into the game by the server. These clients can be viewed as interacting with newly sampled local MAB models. Each client, regardless of newly added or not, performs a sequential arm sampling among the currently active arms for $K_pf(p)$ times on their own local models, which means each active arm is pulled $f(p)$ times by each client. Thus, arm $k \in[K_p]$ is played a total of $M(p)f(p)$ times in phase $p$, where $M(p) =\sum_{q=1}^pg(p)$ is the overall number of clients at phase $p$. Parameters $g(p)$ and $f(p)$ are flexible and we discuss the impact of these choices on the regret in the next section. It is worth noting that the rate of admitting new clients is determined not only by $g(p)$ but also by $f(p)$, which characterizes the frequency of client sampling. With new observations from arm sampling, each client $m$ updates her local estimations, i.e., sample mean $\bar{\mu}_{k,m}(p),k\in[K_p]$, then sends them to the central server as a local parameter update. Note that uploading sample means instead of raw samples benefits the preservation of privacy, and additional methods for better privacy protection are presented in the supplementary material.

After receiving local parameter updates from the clients, the central server first updates the global estimation as the average of them for each active arm, i.e.,  $\bar{\mu}_{k}(p)=\frac{1}{M(p)}\sum_{m=1}^{M(p)}\bar{\mu}_{k,m}(p),k\in[K_p]$. While recognizing two coexisting uncertainties, a ``double'' confidence bound $B_{p,2}$ is adopted to characterize them simultaneously as: 
\begin{equation*}
B_{p,2}=\underbrace{\sqrt{6\sigma^2\eta_p\log\left(T\right)}}_{\text{arm sampling}}+\underbrace{\sqrt{6\sigma_c^2\log\left(T\right)/M(p)}}_{\text{client sampling}},
\end{equation*}
where $\eta_p=\frac{1}{M(p)^2}\sum_{q=1}^p\frac{g(q)}{F(p)-F(q-1)}$ and $F(p)=\sum_{q=1}^p f(q)$ with $F(0)=0$. The first terms in $B_{p,2}$ characterizes the uncertainty from arm sampling, which illustrates the gap between the averaged sampled local model and the exact averaged local model. The second term represents the uncertainty from client sampling, which captures the gap between the exact averaged local model and the (hidden) global model. Note that these two types of uncertainty are not independent of each other, since more admitted clients can perform more pulls on arms, thus reducing both simultaneously.

With the global estimations and the confidence bound, the elimination set $E_p$ is determined by the server, which contains arms that are sub-optimal with a high probability:
\begin{equation*}
 	E_p=\left\{k\in[K_p]\big| \bar{\mu}_k(p)+B_{p,2}\leq \max\nolimits_{l\in[K_p]}\bar{\mu}_l(p)-B_{p,2}\right\}.
\end{equation*}  
The set $[E_p]$ is then sent back to the clients, who then remove these arms from their sets of active arms. This iteration keeps going until there is only one active arm left, which is the optimal arm with a high probability.

\subsection{Regret Analysis}
The regret of the Fed2-UCB algorithm is the combination of the exploration loss and communication loss, and relies on the design of $g(p)$ and $f(p)$.
\begin{theorem}\label{thm:regret_2}
    For $k\not=k_*$,  we denote  $\Delta_k=\mu_*-\mu_k$ and $p_k$ as the smallest integer $p$ such that
	\begin{equation}\label{eqn:regret2_ineq}\small
		96\left(\sigma\sqrt{\eta_p}+\sigma_c/\sqrt{M(p)}\right)^2\log(T)\leq \Delta_k^2,
	\end{equation}
	and $p_{\max}=\max_{k\not=k_*}\{p_k\}$. \shir{If} $\max_{t\leq T}\{M_t\}\leq \beta T$, where $\beta$ is a constant, the regret for the Fed2-UCB algorithm satisfies
	\begin{equation*}\small
	\begin{aligned}
	R_2(T)&\leq \sum\nolimits_{k\not=k_*}\sum\nolimits_{q=1}^{p_k}\Delta_kM(q)f(q)
	+C\sum\nolimits_{q=1}^{p_{\max}}M(q)+4\beta(1+C)K.
	\end{aligned}
	\end{equation*}
\end{theorem}
Eqn. \eqref{eqn:regret2_ineq} describes the requirement for phase $p_k$ under two types of uncertainty, by which the sub-optimal arm $k$ is guaranteed to be eliminated with a high probability. For it to hold, eventually we need at least $O(\log(T))$ clients in the game, which coincides with Eqn.~\eqref{eqn:m_accuracy}.

\congr{Theorem \ref{thm:regret_2} provides a general description, using unspecified $g(p)$ and $f(p)$. A better characterization can be had with more specific choices.}

\begin{corollary}\label{col:regret_2}
    With $f(p)=\kappa$ where $\kappa$ is a constant, and $g(p) = 2^p$, the asymptotic regret of Fed2-UCB is
    \begin{equation*}\small
    \begin{aligned}
	R_2(T)&=O\left(\sum_{k\not=k_*}\frac{\kappa(\sigma/\sqrt{\kappa}+\sigma_c)^2\log(T)}{\Delta_k}+C\frac{(\sigma/\sqrt{\kappa}+\sigma_c)^2\log(T)}{\Delta^2}\right).
	\end{aligned}
    \end{equation*}
\end{corollary}
Corollary \ref{col:regret_2} shows that with carefully designed $f(p)=\kappa$ and $g(p)= 2^p$, Fed2-UCB can achieve a regret of $O(\log(T))$. The exploration loss approaches the single-player MAB lower bound\footnote{In the case with Bernoulli rewards, it can be observed that the two terms are of the same order by invoking the Pinsker's inequality $kl(\mu_i,\mu_j)\geq 2(\mu_j-\mu_i)^2$.} in Eqn.~\eqref{eqn:single_lower} \cite{Lai:1985}, which shows the effectiveness of exploration in Fed2-UCB. Since at least $O(\log(T))$ clients need to be involved as indicated by Eqn. \eqref{eqn:m_accuracy}, an $O(\log(T))$ communication loss achieved in Corollary \ref{col:regret_2} is inevitable, which demonstrates the communication efficiency. The overall regret in Corollary \ref{col:regret_2} proves that Fed2-UCB can effectively deal with two types of uncertainty while balancing the communication loss. 

The choice of $g(p)=2^p$ and $f(p)=\kappa$ leads to an exponentially decreasing $B_{p,2}$, which can be viewed as maintaining an exponentially decreasing estimation $\hat{\Delta}$ of $\Delta$ and eliminating arms with a larger gap \cite{auer2010ucb}; thus, it naturally solves the difficulty associated with the unknown $\Delta$. The regret behavior of several other choices of $f(p)$ and $g(p)$ are given in the supplementary material. 

\section{Special Case: Exact Model and Fed1-UCB}
While the approximate model introduces two types of uncertainty simultaneously, here we study a special case of the \textit{exact model}, where the uncertainty from client sampling does not exist. Correspondingly, the Fed1-UCB algorithm, which degenerates from Fed2-UCB, is designed and analyzed.

\subsection{The Exact Model}
In the exact model, the number of clients is fixed, i.e., $M_t=M$, $\shir{\forall t}$, and the global model is the \textit{exact} average of all the local models, which means the global arm $k$ has a mean reward of $\mu_k=\frac{1}{M}\sum_{m=1}^M\mu_{k,m}$. Thus, the global model can be perfectly reconstructed with information of local models and there only exists the uncertainty from arm sampling. The regret expression can be simplified to $R(T)=\mathbb{E}\left[\sum_{t=1}^T MX_{k_*}(t)-\sum_{t=1}^T\sum_{m=1}^MX_{\pi_m(t)}(t)+CMT_c\right]$.  This model focuses on optimizing the performance for a fixed group of clients that do not change throughout the $T$ time steps. In other words, the global model is not exogenously generated but adapts to the involved clients. Taken the recommender system as an example, the overall popularity of one item is the average of its popularity over the potential clients.

\subsection{The Fed1-UCB Algorithm}
Without the uncertainty from client sampling, there is no need of admitting new clients. The same exploration and communication procedure \congr{of Fed2-UCB} is performed in Fed1-UCB without client admitting. The confidence bound used in arm eliminations is also degenerated from $B_{p,2}$ to $B_{p,1}=\sqrt{{6\sigma^2\log(T)}/{(MF(p))}}$, which only characterizes the uncertainty from arm sampling. A complete description of Fed1-UCB is given in the supplementary material.  

\subsection{Theoretical Analysis}
The regret for the Fed1-UCB algorithm only relies on $f(p)$ and is characterized by the following theorem.
\begin{theorem}\label{thm:regret_1}
    For $k\not=k_*$, we denote $\Delta_k=\mu_*-\mu_k$, \shir{$F(p) = \sum_{q=1}^p f(q)$}, $p_k$ as the smallest integer $p$ such that
	\begin{equation}\label{eqn:regret1_ineq}
		MF(p)\geq 96\sigma^2\log(T)/\Delta_k^2,
	\end{equation}
	and $p_{\max}=\max_{k\not=k_*}\{p_k\}$. The regret of Fed1-UCB satisfies
	\begin{equation*}
	\begin{aligned}
    R_1(T)\leq M\sum\nolimits_{k\not=k_*}\Delta_k F(p_k)+CMp_{\max}+2(1+C)MK.
	\end{aligned}
	\end{equation*}
\end{theorem}
\congr{Somewhat surprisingly, Eqn. \eqref{eqn:regret1_ineq} shows that} although involving more clients leads to a faster convergence (i.e., smaller $p_k$), in general the overall necessary arm pulls performed by the clients, i.e., $MF(p_k)$, are independent of $M$. In other words, we can trade off the convergence time with number of clients without additional exploration loss.
\begin{corollary}\label{col:regret_1}
    With $f(p)=\lceil \kappa\log(T)\rceil$ where $\kappa$ is a constant, the asymptotic regret of the Fed1-UCB algorithm is
       $$ R_1(T) =  O\left(\sum_{k\not=k_*}\frac{\sigma^2\log(T)}{\Delta_k}\right).$$
\end{corollary}
\congr{Corollary \ref{col:regret_1} states that the exploration loss of Fed1-UCB} approaches the single-player MAB lower bound in Eqn.~\eqref{eqn:single_lower} \cite{Lai:1985}. \shir{It is also worth noting that with $f(p)=\lceil \kappa\log(T)\rceil$, the communication loss of Fed-1UCB is a non-dominating constant, which demonstrates its communication efficiency}. Furthermore, the regret is independent of $M$ asymptotically. The regret behavior with other choices of $f(p)$ are discussed in the supplementary material.

\section{Experiments}
Numerical experiments have been carried out under both applications of cognitive radio and recommender system. Their results are reported in this section to demonstrate the effectiveness and efficiency of Fed2-UCB and Fed1-UCB. For the cognitive radio example, due to the lack of suitable real-world datasets, synthetic datasets are used for simulations \cite{avner2014concurrent,bande2019multi}. For the recommender system, real-world evaluations are performed. The performance of a (hypothetical) single-player improved UCB algorithm \cite{auer2010ucb} directly performed at the server is used as the baseline (labeled as ``baseline''). The communication cost is set to be $C=1$.

\begin{figure*}[htb]
    \centering
	\begin{minipage}[t]{0.245\linewidth}
		\centering
		\includegraphics[width=\linewidth]{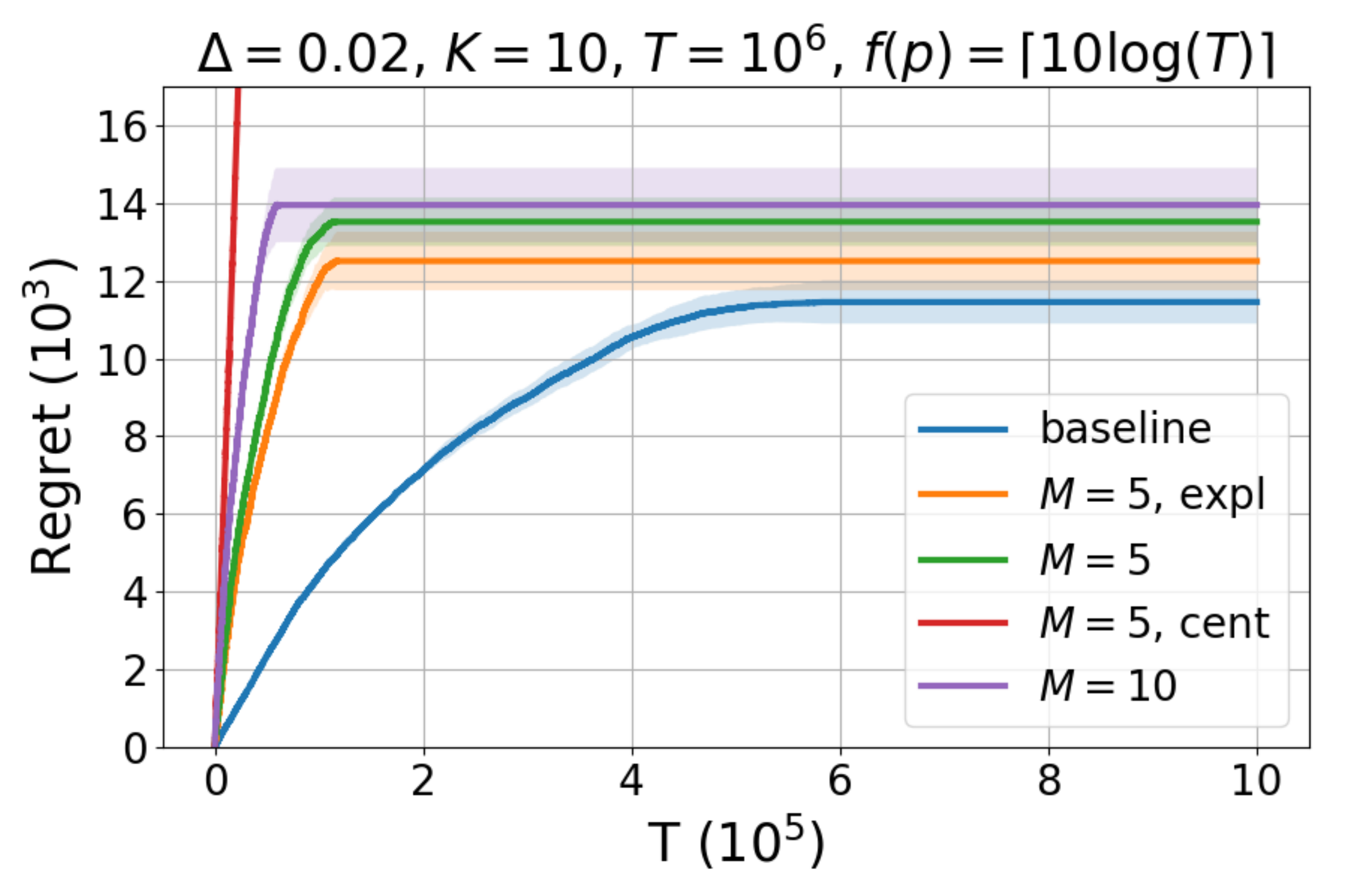}
		\caption{Fed1-UCB performance.}
		\label{fig:feducb_perf}
	\end{minipage}
	\begin{minipage}[t]{0.245\linewidth}
		\centering
		\includegraphics[width=\linewidth]{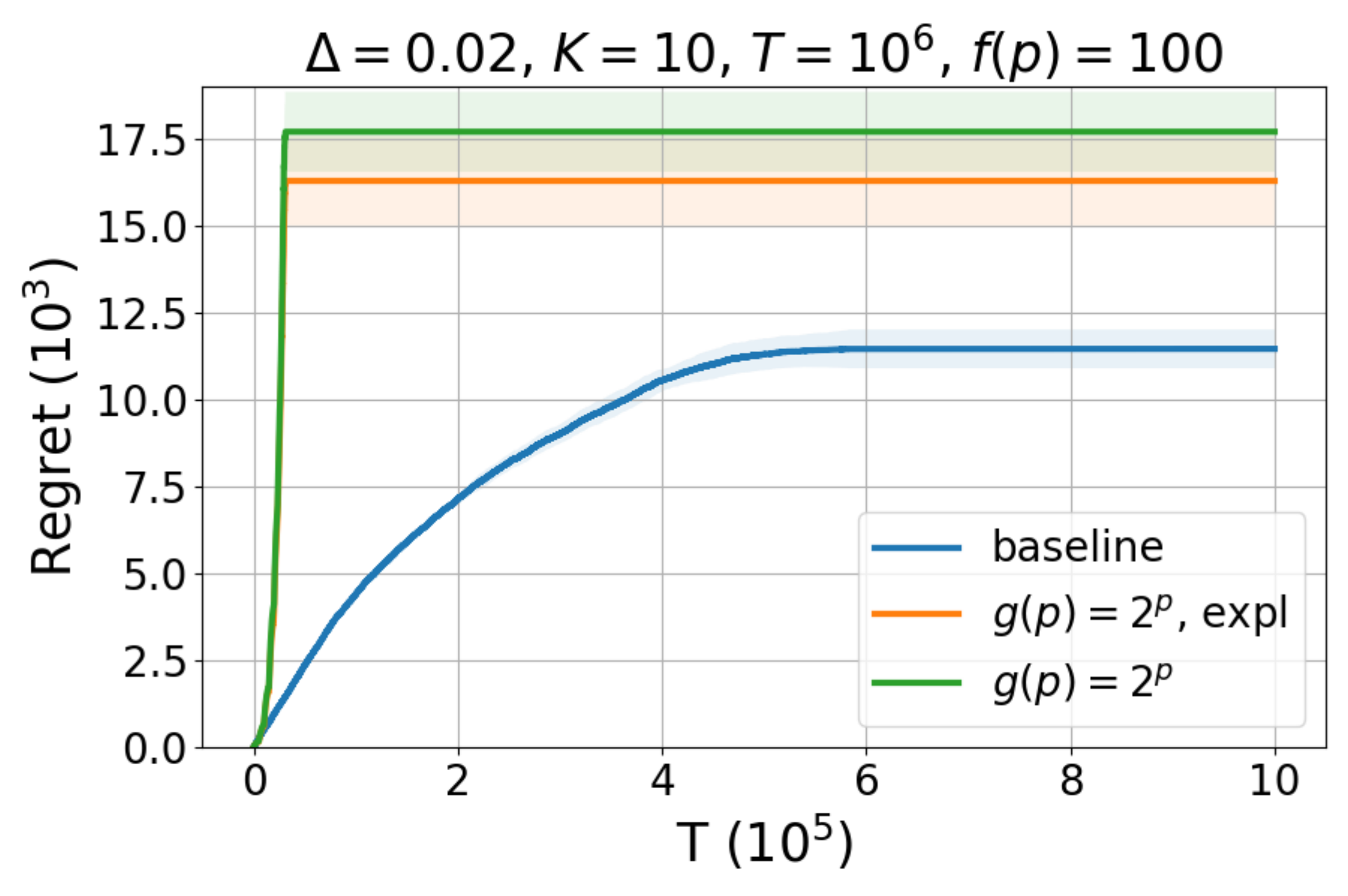}
		\caption{Fed2-UCB performance.}
		\label{fig:fed2ucb_perf}
	\end{minipage}
	\begin{minipage}[t]{0.245\linewidth}
		\centering
		\includegraphics[width=\linewidth]{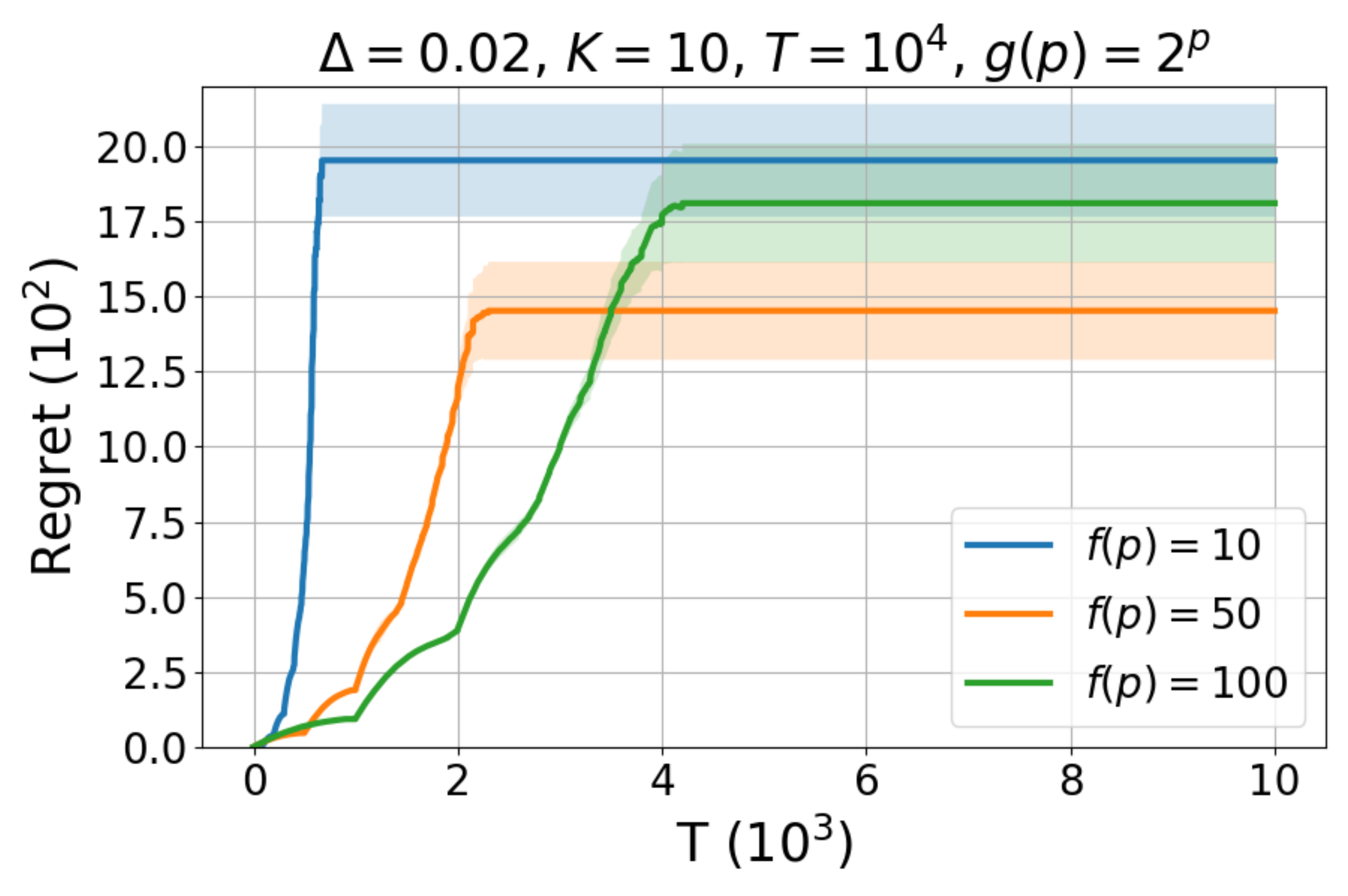}
		\caption{Fed2-UCB with $f(p)$.}
		\label{fig:fed2ucb_perf_short}
	\end{minipage}
	\begin{minipage}[t]{0.245\linewidth}
		\centering
		\includegraphics[width=\linewidth]{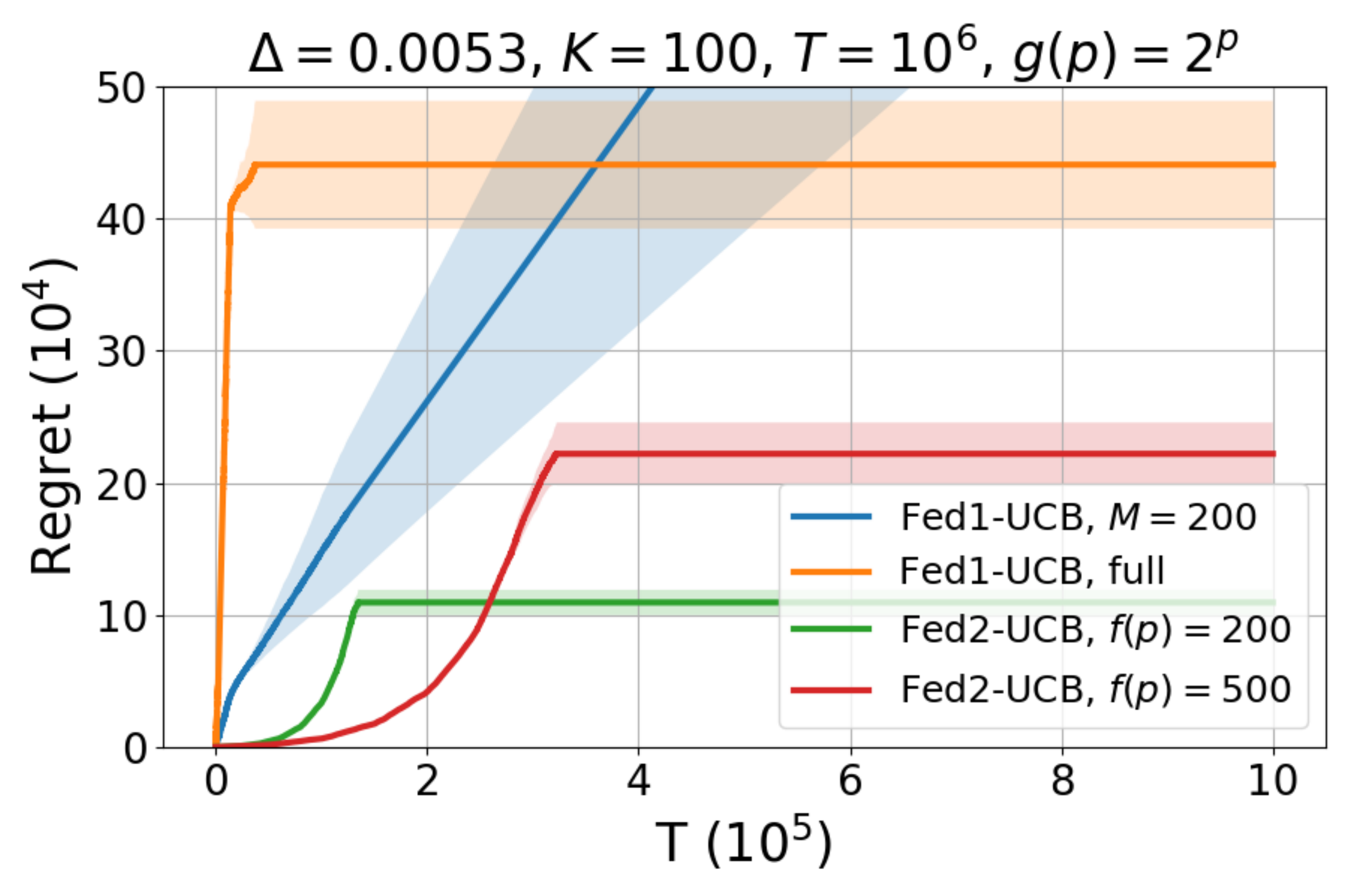}
		\caption{MovieLens.}
		\label{fig:movielens_perf}
	\end{minipage}
\end{figure*}

\subsection{Synthetic Dataset for Cognitive Radio}
A bandit game with $K=10$ arms is used to mimic $10$ candidate channels, and Gaussian distributions with $\sigma = 0.5$ are used to generate local observations of the channel availability. The means of global arms are in the interval $[0.7,0.8]$ with $\Delta=0.02$. We first start with the relatively simple exact model, where $M=5$ clients are involved while arm $1$ is not the optimal arm of any of their local models. As shown in Fig.~\ref{fig:feducb_perf}, with $f(p)=\lceil10\log(T)\rceil$, if there is no communication loss, Fed1-UCB (labeled as ``expl'') achieves almost the same performance as the baseline, which proves its effectiveness. When considering the communication loss, the centralized version of Fed1-UCB (labeled as ``cent''), where clients send their raw data in every time slot, has a very large regret due to significant yet unnecessary communications. However, with $f(p)=\lceil10\log(T)\rceil$, Fed1-UCB only incurs a small communication loss, which proves its efficiency. It is also worth noting that Fed1-UCB converges faster than the baseline, which is the result of \congr{higher arm sampling rate} due to multiple clients simultaneously pulling arms. In other words, the fast convergence over time is due to the increased client dimension. When increasing the number of clients to $M=10$, the overall regret remains approximately the same as $M=5$, but with even faster convergence, which corroborates Theorem \ref{thm:regret_1} and Corollary \ref{col:regret_1}.

For the approximate model, the same set of global arms is used while the local models are generated by Gaussian distributions with $\sigma_c=0.02$. Fig.~\ref{fig:fed2ucb_perf} shows that Fed2-UCB with $f(p)=100$ and $g(p)=2^p$ successfully finds the optimal global arm and, without communication loss, has a performance (labelled as ``expl'') slightly \congr{worse than} the baseline. Furthermore, the additional communication loss is very limited. Compared with the performance of Fed1-UCB in Fig.~\ref{fig:feducb_perf}, we see that Fed2-UCB achieves almost the same performance for the more challenging approximate model, and the convergence of Fed2-UCB is even faster since the impact of increasing the number of clients is already significant at the very beginning. Under a reduced time horizon $T=10^4$, Fig. \ref{fig:fed2ucb_perf_short} \congr{provides a finer look at} the shape of regret curves of Fed2-UCB and illustrates the need of balancing two types of uncertainty\footnote{The baseline is not included in Fig. \ref{fig:fed2ucb_perf_short} since it cannot converge in such a short time period.}. With a short update period $f(p)=10$, new clients are admitted rapidly, which sharply decreases the uncertainty from client sampling, but insufficient local exploration leads to a large uncertainty from arm sampling, which causes a large regret despite the fast the convergence. On the other extreme, although local exploration is guaranteed to be sufficient with $f(p)=100$, it admits new clients slowly, which delays the convergence and causes unnecessary local explorations. $f(p)=50$ strikes a better balance between two types of uncertainty and thus a better performance.

\subsection{Real-world Dataset for Recommender System}
The MovieLens dataset \cite{Cantador:RecSys2011} is used for the real-world evaluation as an implementation of recommender system, which has been widely adopted in MAB studies \cite{oh2019thompson,mahadik2020fast}. It links the movies of MovieLens dataset with IMDb and Rotten Tomatoes movie review systems, and contains $2113$ clients and $10197$ movies. All the users are assumed to be available while the movies are randomly divided into $100$ groups and the observations for clients are defined as their ratings of each group of movies. The suboptimality gap of the pre-processed data is $\Delta\approx 0.0053$. The number of arms and potential clients are much larger than the synthetic dataset. First, as shown in Fig.~\ref{fig:movielens_perf}, if a small fraction of clients ($M=200$) are used for Fed1-UCB, which can be viewed as only involving a small amount of clients at the beginning of Fed2-UCB, the regret curve trends upward, meaning the global optimal arm is not found due to insufficient client sampling. Oppositely, when all clients are involved, Fed1-UCB converges to the optimal arm but with a large regret, which shows the harm of oversampling. Using Fed2-UCB and $f(p)=200$ with $g(p)=2^p$, a much better performance is achieved, since only the necessary amount of clients are sampled \congr{to capture} the global model faithfully without unnecessary loss. With $f(p)=500$, new clients are admitted \congr{more slowly} but it still outperforms Fed1-UCB.

\section{Related Works}
\label{sec:related}
Federated learning was introduced by \citet{mcmahan2017communication,konevcny2016federated2,konevcny2016federated} 
and has been an active research topic with studies spanning from communication efficiency \cite{konevcny2016federated,sattler2019robust}, security and privacy \cite{geyer2017differentially,bagdasaryan2020backdoor}, fairness \cite{li2020Fair},  to system designs \cite{smith2017federated,bonawitz2019towards}, with successful applications in recommender system \cite{ammad2019federated}, and medical treatment \cite{li2019privacy}. 

Multi-armed bandits \cite{Auer:2002,lattimore2020bandit} is a rich research area with various applications, such as cognitive radio \cite{gai2010learning}, recommender system \cite{li2010contextual, wu2017returning}, and clinical trials \cite{shen2020learning,lee2020contextual,lee2021aistats}. The research of distributed multi-player MAB (MP-MAB) are related to the proposed FMAB but with fundamental differences. The MP-MAB research either considers the ``cooperative''  setting, where players interact with a \textit{common} MAB game and communicate with each other to accelerate learning \cite{landgren2016distributed,landgren2018social,martinez2019decentralized,wang2019distributed}, or the ``competitive'' setting, where there is no explicit communications and solving arm collisions is the fundamental difficulty \cite{liu2010distributed,avner2014concurrent,rosenski2016multi,Boursier2019,Shi2020aistats,bubeck2020non,shi2020no}. \shir{The FMAB framework differs from MP-MAB research in that clients can only access observation signals locally, which may not be sufficient to infer the true rewards and the optimal global arm. This difference results in a more fundamental role of communications.} 

The concept of \emph{federated bandits} has been touched upon by a few recent works but with very different focuses than our work. Strictly IID local models are studied by \citet{li2020federated2,dubey2020differentially}  with a focus on privacy protection. \citet{agarwal2020federated} study regression-based contextual bandits as an example of the federated residual learning framework, which does not generalize to the setting of our work. \citet{zhu2021federated} consider a similar problem as the exact model and focuses on sharing information through gossiping among clients with privacy protection. \citet{shi2021aistats} consider federated bandits with personalization, where the clients play mixed bandit games that incorporate both global and local models.

\section{Conclusion}
In this work, we have developed a general FMAB framework that bridges MAB and FL. In the proposed approximate model, a new source of uncertainty from client sampling was introduced. We proposed the Fed2-UCB algorithm that involves clients in an increasing manner and explores two types of uncertainty simultaneously while balancing the communication loss, which achieves an  $O(\log(T))$ regret. A special case of the exact model was studied with the Fed1-UCB algorithm, which also achieves an order-optimal regret while providing an opportunity to tradeoff the convergence time and number of clients. Experiments with synthetic and real-world datasets proved the effectiveness of the proposed algorithms and corroborated the theoretical analysis.

\section*{Acknowledgements}

The work was partially supported by the US National Science Foundation (NSF) under Grant CNS-2002902, ECCS-2033671 and ECCS-2029978.

\section*{Ethics Statement}
Federated multi-armed bandits represents a novel framework that extends the federated learning principles to the bandit problems. Intellectually, this work broadens the scope of multi-player MAB by proposing global-local (non-IID) bandit model interactions through clients in a communication-efficient and privacy-preserving way. This new FMAB framework is general and may spark other innovations in this field. In addition, two FMAB models are considered that capture different global-local model relationships, and their corresponding algorithms and theoretical analyses may be useful for other similar problems. Practically, this work has the potential to benefit applications with sensitive data collected locally at the clients while a global bandit model is desired to be learned based on (but cannot directly access) the local data, including but not limited to the examples of cognitive radio and recommender systems given in the paper. In both examples, our research will enable the central server to avoid collecting raw data from the end users, hence mitigating risks to privacy and cost of communication. This work also provides a tangible way to leverage massively distributed clients for bandit learning. The authors do not foresee significant disadvantage for the involved parties with the proposed FMAB paradigm. As for the failure of the system, it is less likely to be an issue since the operations are fully distributed among large number of clients, but Fed2-UCB and Fed1-UCB may end up finding sub-optimal arms when it does happen. It is, however, a significant risk if the synchronization is broken, which poses an interesting research problem for future investigation. The FMAB framework also requires reliable communication infrastructure and sufficient computation power at the clients, which can be difficult in some situations (e.g., remote/rural areas or low-cost devices).

\bibliography{ref}

\begin{thebibliography}{43}
\providecommand{\natexlab}[1]{#1}
\providecommand{\url}[1]{\texttt{#1}}
\providecommand{\urlprefix}{URL }
\expandafter\ifx\csname urlstyle\endcsname\relax
  \providecommand{\doi}[1]{doi:\discretionary{}{}{}#1}\else
  \providecommand{\doi}{doi:\discretionary{}{}{}\begingroup
  \urlstyle{rm}\Url}\fi

\bibitem[{3GPP(2020)}]{3gpp:36921}
3GPP. 2020.
\newblock {Evolved Universal Terrestrial Radio Access (E-UTRA); FDD Home eNode
  B (HeNB) Radio Frequency (RF) requirements analysis}.
\newblock TR {36.921}, {3GPP}.

\bibitem[{Agarwal, Langford, and Wei(2020)}]{agarwal2020federated}
Agarwal, A.; Langford, J.; and Wei, C.-Y. 2020.
\newblock Federated residual learning.
\newblock \emph{arXiv preprint} arXiv:2003.12880.

\bibitem[{Ammad-Ud-Din et~al.(2019)Ammad-Ud-Din, Ivannikova, Khan, Oyomno, Fu,
  Tan, and Flanagan}]{ammad2019federated}
Ammad-Ud-Din, M.; Ivannikova, E.; Khan, S.~A.; Oyomno, W.; Fu, Q.; Tan, K.~E.;
  and Flanagan, A. 2019.
\newblock Federated collaborative filtering for privacy-preserving personalized
  recommendation system.
\newblock \emph{arXiv preprint} arXiv:1901.09888.

\bibitem[{Auer, Cesa-Bianchi, and Fischer(2002)}]{Auer:2002}
Auer, P.; Cesa-Bianchi, N.; and Fischer, P. 2002.
\newblock Finite-time analysis of the multiarmed bandit problem.
\newblock \emph{Mach. Learn.} 47(2-3): 235--256.

\bibitem[{Auer and Ortner(2010)}]{auer2010ucb}
Auer, P.; and Ortner, R. 2010.
\newblock {UCB} revisited: Improved regret bounds for the stochastic
  multi-armed bandit problem.
\newblock \emph{Periodica Mathematica Hungarica} 61(1-2): 55--65.

\bibitem[{Avner and Mannor(2014)}]{avner2014concurrent}
Avner, O.; and Mannor, S. 2014.
\newblock Concurrent bandits and cognitive radio networks.
\newblock In \emph{Joint European Conference on Machine Learning and Knowledge
  Discovery in Databases}, 66--81. Springer.

\bibitem[{Bagdasaryan et~al.(2020)Bagdasaryan, Veit, Hua, Estrin, and
  Shmatikov}]{bagdasaryan2020backdoor}
Bagdasaryan, E.; Veit, A.; Hua, Y.; Estrin, D.; and Shmatikov, V. 2020.
\newblock How to backdoor federated learning.
\newblock In \emph{International Conference on Artificial Intelligence and
  Statistics}, 2938--2948. PMLR.

\bibitem[{Bande and Veeravalli(2019)}]{bande2019multi}
Bande, M.; and Veeravalli, V.~V. 2019.
\newblock Multi-user multi-armed bandits for uncoordinated spectrum access.
\newblock In \emph{International Conference on Computing, Networking and
  Communications (ICNC)}, 653--657. IEEE.

\bibitem[{Bonawitz et~al.(2019)Bonawitz, Eichner, Grieskamp, Huba, Ingerman,
  Ivanov, Kiddon, Konecny, Mazzocchi, McMahan, Overveldt, Petrou, Ramage, and
  Roselander}]{bonawitz2019towards}
Bonawitz, K.; Eichner, H.; Grieskamp, W.; Huba, D.; Ingerman, A.; Ivanov, V.;
  Kiddon, C.; Konecny, J.; Mazzocchi, S.; McMahan, H.~B.; Overveldt, T.~V.;
  Petrou, D.; Ramage, D.; and Roselander, J. 2019.
\newblock Towards federated learning at scale: System design.
\newblock In \emph{Proceedings of the 2nd SysML Conference}, 1--15.

\bibitem[{Boursier and Perchet(2019)}]{Boursier2019}
Boursier, E.; and Perchet, V. 2019.
\newblock {SIC-MMAB}: synchronisation involves communication in multiplayer
  multi-Armed bandits.
\newblock In \emph{Advances in Neural Information Processing Systems},
  12071--12080.

\bibitem[{Bubeck and Cesa-Bianchi(2012)}]{Bubeck:2012}
Bubeck, S.; and Cesa-Bianchi, N. 2012.
\newblock Regret analysis of stochastic and nonstochastic multi-armed bandit
  problems.
\newblock \emph{Foundations and Trends in Machine Learning} 5(1): 1--122.

\bibitem[{Bubeck et~al.(2020)Bubeck, Li, Peres, and Sellke}]{bubeck2020non}
Bubeck, S.; Li, Y.; Peres, Y.; and Sellke, M. 2020.
\newblock Non-stochastic multi-player multi-armed bandits: Optimal rate with
  collision information, sublinear without.
\newblock In \emph{Conference on Learning Theory}, 961--987. PMLR.

\bibitem[{Cantador, Brusilovsky, and Kuflik(2011)}]{Cantador:RecSys2011}
Cantador, I.; Brusilovsky, P.; and Kuflik, T. 2011.
\newblock {2nd Workshop on Information Heterogeneity and Fusion in Recommender
  Systems (HetRec 2011)}.
\newblock In \emph{Proceedings of the 5th ACM Conference on Recommender
  Systems}, RecSys 2011. New York, NY, USA: ACM.

\bibitem[{Dubey and Pentland(2020)}]{dubey2020differentially}
Dubey, A.; and Pentland, A. 2020.
\newblock Differentially-private federated linear bandits.
\newblock \emph{Advances in Neural Information Processing Systems} 33.

\bibitem[{Gai, Krishnamachari, and Jain(2010)}]{gai2010learning}
Gai, Y.; Krishnamachari, B.; and Jain, R. 2010.
\newblock Learning multiuser channel allocations in cognitive radio networks: A
  combinatorial multi-armed bandit formulation.
\newblock In \emph{IEEE Symposium on New Frontiers in Dynamic Spectrum
  (DySPAN)}, 1--9. IEEE.

\bibitem[{Geyer, Klein, and Nabi(2017)}]{geyer2017differentially}
Geyer, R.~C.; Klein, T.; and Nabi, M. 2017.
\newblock Differentially private federated learning: A client level
  perspective.
\newblock \emph{arXiv preprint} arXiv:1712.07557.

\bibitem[{Kone{\v{c}}n{\`y} et~al.(2016{\natexlab{a}})Kone{\v{c}}n{\`y},
  McMahan, Ramage, and Richt{\'a}rik}]{konevcny2016federated2}
Kone{\v{c}}n{\`y}, J.; McMahan, H.~B.; Ramage, D.; and Richt{\'a}rik, P.
  2016{\natexlab{a}}.
\newblock Federated optimization: Distributed machine learning for on-device
  intelligence.
\newblock \emph{arXiv preprint} arXiv:1610.02527.

\bibitem[{Kone{\v{c}}n{\`y} et~al.(2016{\natexlab{b}})Kone{\v{c}}n{\`y},
  McMahan, Yu, Richt{\'a}rik, Suresh, and Bacon}]{konevcny2016federated}
Kone{\v{c}}n{\`y}, J.; McMahan, H.~B.; Yu, F.~X.; Richt{\'a}rik, P.; Suresh,
  A.~T.; and Bacon, D. 2016{\natexlab{b}}.
\newblock Federated learning: Strategies for improving communication
  efficiency.
\newblock In \emph{Advances in Neural Information Processing Systems --
  Workshop on Private Multi-Party Machine Learning}.

\bibitem[{Lai and Robbins(1985)}]{Lai:1985}
Lai, T.~L.; and Robbins, H. 1985.
\newblock Asymptotically efficient adaptive allocation rules.
\newblock \emph{Adv. Appl. Math.} 6(1): 4--22.

\bibitem[{Landgren, Srivastava, and Leonard(2016)}]{landgren2016distributed}
Landgren, P.; Srivastava, V.; and Leonard, N.~E. 2016.
\newblock On distributed cooperative decision-making in multiarmed bandits.
\newblock In \emph{European Control Conference (ECC)}, 243--248. IEEE.

\bibitem[{Landgren, Srivastava, and Leonard(2018)}]{landgren2018social}
Landgren, P.; Srivastava, V.; and Leonard, N.~E. 2018.
\newblock Social imitation in cooperative multiarmed bandits: partition-based
  algorithms with strictly local information.
\newblock In \emph{IEEE Conference on Decision and Control (CDC)}, 5239--5244.
  IEEE.

\bibitem[{Lattimore and Szepesv{\'a}ri(2020)}]{lattimore2020bandit}
Lattimore, T.; and Szepesv{\'a}ri, C. 2020.
\newblock \emph{Bandit algorithms}.
\newblock Cambridge University Press.

\bibitem[{Lee et~al.(2020)Lee, Shen, Jordon, and van~der
  Schaar}]{lee2020contextual}
Lee, H.-S.; Shen, C.; Jordon, J.; and van~der Schaar, M. 2020.
\newblock Contextual constrained learning for dose-finding clinical trials.
\newblock In \emph{Proceedings of the 23rd International Conference on
  Artificial Intelligence and Statistics (AISTATS)}, 2645--2654.

\bibitem[{Lee et~al.(2021)Lee, Shen, Zame, Lee, and van~der
  Schaar}]{lee2021aistats}
Lee, H.-S.; Shen, C.; Zame, W.; Lee, J.; and van~der Schaar, M. 2021.
\newblock {SDF-Bayes}: Cautious Optimism in Safe Dose-Finding Clinical Trials
  with Drug Combinations and Heterogeneous Patient Groups.
\newblock In \emph{Proceedings of the 24rd International Conference on
  Artificial Intelligence and Statistics (AISTATS)}.

\bibitem[{Li et~al.(2010)Li, Chu, Langford, and Schapire}]{li2010contextual}
Li, L.; Chu, W.; Langford, J.; and Schapire, R.~E. 2010.
\newblock A contextual-bandit approach to personalized news article
  recommendation.
\newblock In \emph{Proceedings of the 19th International Conference on World
  Wide Web}, 661--670.

\bibitem[{Li et~al.(2020)Li, Sanjabi, Beirami, and Smith}]{li2020Fair}
Li, T.; Sanjabi, M.; Beirami, A.; and Smith, V. 2020.
\newblock Fair resource allocation in federated learning.
\newblock In \emph{International Conference on Learning Representations}.

\bibitem[{Li, Song, and Fragouli(2020)}]{li2020federated2}
Li, T.; Song, L.; and Fragouli, C. 2020.
\newblock Federated recommendation system via differential privacy.
\newblock \emph{IEEE International Symposium on Information Theory (ISIT)} .

\bibitem[{Li et~al.(2019)Li, Milletar{\`\i}, Xu, Rieke, Hancox, Zhu, Baust,
  Cheng, Ourselin, Cardoso et~al.}]{li2019privacy}
Li, W.; Milletar{\`\i}, F.; Xu, D.; Rieke, N.; Hancox, J.; Zhu, W.; Baust, M.;
  Cheng, Y.; Ourselin, S.; Cardoso, M.~J.; et~al. 2019.
\newblock Privacy-preserving federated brain tumour segmentation.
\newblock In \emph{International Workshop on Machine Learning in Medical
  Imaging}, 133--141. Springer.

\bibitem[{Liu and Zhao(2010)}]{liu2010distributed}
Liu, K.; and Zhao, Q. 2010.
\newblock Distributed learning in multi-armed bandit with multiple players.
\newblock \emph{IEEE Transactions on Signal Processing} 58(11): 5667--5681.

\bibitem[{Mahadik et~al.(2020)Mahadik, Wu, Li, and Sabne}]{mahadik2020fast}
Mahadik, K.; Wu, Q.; Li, S.; and Sabne, A. 2020.
\newblock Fast distributed bandits for online recommendation systems.
\newblock In \emph{Proceedings of the 34th ACM International Conference on
  Supercomputing}, 1--13.

\bibitem[{Mart{\'\i}nez-Rubio, Kanade, and
  Rebeschini(2019)}]{martinez2019decentralized}
Mart{\'\i}nez-Rubio, D.; Kanade, V.; and Rebeschini, P. 2019.
\newblock Decentralized cooperative stochastic bandits.
\newblock In \emph{Advances in Neural Information Processing Systems},
  4529--4540.

\bibitem[{McMahan et~al.(2017)McMahan, Moore, Ramage, Hampson, and
  y~Arcas}]{mcmahan2017communication}
McMahan, B.; Moore, E.; Ramage, D.; Hampson, S.; and y~Arcas, B.~A. 2017.
\newblock Communication-efficient learning of deep networks from decentralized
  data.
\newblock In \emph{Proceedings of the 20th International Conference on
  Artificial Intelligence and Statistics (AISTATS)}, 1273--1282. Fort
  Lauderdale, FL, USA.

\bibitem[{Oh and Iyengar(2019)}]{oh2019thompson}
Oh, M.-h.; and Iyengar, G. 2019.
\newblock Thompson sampling for multinomial logit contextual bandits.
\newblock In \emph{Advances in Neural Information Processing Systems},
  3151--3161.

\bibitem[{Rosenski, Shamir, and Szlak(2016)}]{rosenski2016multi}
Rosenski, J.; Shamir, O.; and Szlak, L. 2016.
\newblock Multi-player bandits -- a musical chairs approach.
\newblock In \emph{International Conference on Machine Learning}, 155--163.

\bibitem[{Sattler et~al.(2019)Sattler, Wiedemann, M{\"u}ller, and
  Samek}]{sattler2019robust}
Sattler, F.; Wiedemann, S.; M{\"u}ller, K.-R.; and Samek, W. 2019.
\newblock Robust and communication-efficient federated learning from non-iid
  data.
\newblock \emph{IEEE Transactions on Neural Networks and Learning Systems} .

\bibitem[{Shen et~al.(2020)Shen, Wang, Villar, and van~der
  Schaar}]{shen2020learning}
Shen, C.; Wang, Z.; Villar, S.; and van~der Schaar, M. 2020.
\newblock Learning for dose allocation in adaptive clinical trials with safety
  constraints.
\newblock In \emph{Proceedings of the 37th International Conference on Machine
  Learning (ICML)}, 7310--7320.

\bibitem[{Shi and Shen(2020)}]{shi2020no}
Shi, C.; and Shen, C. 2020.
\newblock On no-sensing adversarial multi-player multi-armed bandits with
  collision communications.
\newblock \emph{arXiv preprint} arXiv:2011.01090.

\bibitem[{Shi, Shen, and Yang(2021)}]{shi2021aistats}
Shi, C.; Shen, C.; and Yang, J. 2021.
\newblock Federated Multi-armed Bandits with Personalization.
\newblock In \emph{Proceedings of the 24rd International Conference on
  Artificial Intelligence and Statistics (AISTATS)}.

\bibitem[{{Shi} et~al.(2020){Shi}, {Xiong}, {Shen}, and
  {Yang}}]{Shi2020aistats}
{Shi}, C.; {Xiong}, W.; {Shen}, C.; and {Yang}, J. 2020.
\newblock Decentralized multi-player multi-armed bandits with no collision
  information.
\newblock In \emph{Proceedings of the 23rd International Conference on
  Artificial Intelligence and Statistics (AISTATS)}. Palermo, Sicily, Italy.

\bibitem[{Smith et~al.(2017)Smith, Chiang, Sanjabi, and
  Talwalkar}]{smith2017federated}
Smith, V.; Chiang, C.-K.; Sanjabi, M.; and Talwalkar, A.~S. 2017.
\newblock Federated multi-task learning.
\newblock In \emph{Advances in Neural Information Processing Systems},
  4424--4434.

\bibitem[{Wang et~al.(2020)Wang, Hu, Chen, and Wang}]{wang2019distributed}
Wang, Y.; Hu, J.; Chen, X.; and Wang, L. 2020.
\newblock Distributed bandit learning: Near-optimal regret with efficient
  communication.
\newblock In \emph{International Conference on Learning Representations}.

\bibitem[{Wu et~al.(2017)Wu, Wang, Hong, and Shi}]{wu2017returning}
Wu, Q.; Wang, H.; Hong, L.; and Shi, Y. 2017.
\newblock Returning is believing: Optimizing long-term user engagement in
  recommender systems.
\newblock In \emph{Proceedings of the 2017 ACM on Conference on Information and
  Knowledge Management}, 1927--1936.

\bibitem[{Zhu et~al.(2021)Zhu, Zhu, Liu, and Liu}]{zhu2021federated}
Zhu, Z.; Zhu, J.; Liu, J.; and Liu, Y. 2021.
\newblock Federated bandit: A gossiping approach.
\newblock \emph{Proceedings of the ACM on Measurement and Analysis of Computing
  Systems} 5(1): 1--29.

\end{thebibliography}

\newpage
\appendix
\onecolumn
\newpage
\appendix

\section{Similarities between FL and FMAB}
Table 1 summarizes the main similarities between the proposed FMAB framework and the core principles in FL. \congr{There is a strong and natural correspondence between these two frameworks, which share many key features.}
\begin{table*}[thb]
	\centering
	\begin{tabular}{ll}
		\toprule[1pt]
		Federated learning     & Federated MAB \\
		\midrule[0.5pt]
		Non-IID local datasets	& Non-IID local bandit games\\
		Massively distributed &  Client sampling in the approximate model\\
		Privacy preservation & Sharing sample means instead of raw samples\\
		Communication efficiency & Controlling the communication loss\\
		Goal: increase global model accuracy & Goal: minimize regret Eqn.~\eqref{eqn:regret_fed}\\
		\bottomrule[1pt]
	\end{tabular}\\
	\caption{Similarities between federated MAB and federated learning}
	\label{tbl:concept_sim}
\end{table*}

\section{Fed1-UCB Algorithm Description}
For completeness, the detailed algorithm descriptions of Fed1-UCB, for both clients and server, are given in Algorithm \ref{alg:fed_local} and \ref{alg:fed_central}, respectively.

\begin{minipage}[t]{.49\textwidth}
	\begin{algorithm}[H]
		\small
		\caption{Fed1-UCB: client $m$}
		\label{alg:fed_local}
		\begin{algorithmic}[1]
			\Require $T$, $K$
			\State Initialize $p\gets 1$; $[K_1]\gets [K]$
			\While{$K_{p}>1$}
			\State \shir{Pull each active arm $k\in [K_p]$ for $f(p)$ times}
			\State \shir{Calculate the local sample means $\bar{\mu}_{k,m}(p),\forall k\in[K_p]$}
			\State Send $\forall k\in[K_p],\bar{\mu}_{k,m}(p)$ to the server
			\State Receive global update set $E_p$ from the server
			\State Update model $[K_{p+1}]\gets[K_p]\backslash E_p$
			\State $p\gets p+1$
			\EndWhile
			\State $F\gets$the only element in $[K_p]$
			\State Stay on arm $F$ until $T$
		\end{algorithmic}
	\end{algorithm}
\end{minipage}%
\hspace{0.01in}
\begin{minipage}[t]{.49\textwidth}
	\begin{algorithm}[H]
		\small
		\caption{Fed1-UCB: central server}
		\label{alg:fed_central}
		\begin{algorithmic}[1]
			\Require $T$, $K$, \shir{$M$}
			\State Initialize $p\gets 1$; $[K_1]\gets [K]$			
			\While{$K_{p}>1$}
			\State Receive local updates $\bar{\mu}_{k,m}(p), \forall k\in[K_p],\forall m\in[M]$
			\State Calculate $\forall k\in[K_p],\bar{\mu}_{k}(p)=\sum_{m=1}^{M}\bar{\mu}_{k,m}(p)/M$
			\State $E_p\gets$ all $k\in[K_p]$ satisfies \Comment{\textit{Arm elimination}}
			$$\bar{\mu}_k(p)+B_{p,1}\leq \max_{l\in[K_p]}\bar{\mu}_l(p)-B_{p,1}$$
			\State Send global update set $E_p$ to all the clients
			\State Update model $[K_{p+1}]\gets[K_p]\backslash E_p$
			\State $p\gets p+1$
			\EndWhile
		\end{algorithmic}
	\end{algorithm}
\end{minipage}

\section{Additional Discussions on the Regret Analysis}
The regrets in Theorem~\ref{thm:regret_2} and Theorem~\ref{thm:regret_1} are related to the choices at the server, i.e., $f(p)$ and $g(p)$. In this section, we provide a more detailed discussion of the impact of these choices on the regret.
\subsection{Discussion for Theorem~\ref{thm:regret_2}}\label{app:dis_regret2}
From Eqns.~\eqref{eqn:m_accuracy}, \eqref{eqn:m_upper} and \eqref{eqn:regret2_ineq}, there are $\Theta(\log(T))$ clients involved in the game eventually, which means that the choice of any $f(p)$ with an order higher than $O(1)$ cannot have an $O(\log(T))$ regret. Also, with $O(\log(T))$ involved clients, a constant communication loss is no longer achievable as shown in Corollary \ref{col:regret_1}. Thus, we focus on choices of $g(p)$ while fixing $f(p)=\kappa$, and the results are given in Table \ref{tbl:regret_Fed2-UCB}. With a linear growth rate $g(p)=\lambda$, we can see that the overall regret is of order $O(\log^2(T))$. While increasing the rate to $g(p)=\left\lceil\lambda\log(T)\right\rceil$, an $O(\log(T))$ regret is achieved but the multiplicative factors are far from optimal. By exponentially increasing involving players with $g(p)=2^p$ or $g(p)=\left\lceil2^p\log(T)\right\rceil$, an exploration loss approaching the lower bound can be achieved, while the communication loss remains sublinear of order $O(\log(T))$.

\begin{table}[htb]
	\centering
	\begin{tabular}{lll}
		\toprule
		$g(p)$        & $p_k$, $k\not=k_*$  & $R_2(T)$\\
		\midrule
		\smallskip
		$\lambda$   & $\left\lceil\frac{96(\sigma/\sqrt{\kappa}+\sigma_c)^2\log(T)}{\lambda(\mu_*-\mu_{k})^2}\right\rceil$ & $O\left(\sum_{k\not=k_*}\frac{\kappa(\sigma/\sqrt{\kappa}+\sigma_c)^4\log^2(T)}{\lambda(\mu_*-\mu_{k})^3}+C\frac{(\sigma/\sqrt{\kappa}+\sigma_c)^4\log^2(T)}{\lambda\Delta^4}\right)$   \\
		\smallskip
		$\left\lceil\lambda\log(T)\right\rceil$   & $\left\lceil\frac{96(\sigma/\sqrt{\kappa}+\sigma_c)^2}{\lambda(\mu_*-\mu_{k})^2}\right\rceil$ & $O\left(\sum_{k\not=k_*}\frac{\kappa(\sigma/\sqrt{\kappa}+\sigma_c)^4\log(T)}{\lambda(\mu_*-\mu_{k})^3}+C\frac{(\sigma/\sqrt{\kappa}+\sigma_c)^4\log(T)}{\Delta^4}\right)$   \\
		\smallskip
		$2^{p}$   & $\left\lceil\log\left(\frac{96(\sigma/\sqrt{\kappa}+\sigma_c)^2\log(T)}{(\mu_*-\mu_{k})^2}\right)\right\rceil$ & $O\left(\sum_{k\not=k_*}\frac{\kappa(\sigma/\sqrt{\kappa}+\sigma_c)^2\log(T)}{(\mu_*-\mu_{k})}+C\frac{(\sigma/\sqrt{\kappa}+\sigma_c)^2\log(T)}{\Delta^2}\right)$   \\
		$\left\lceil2^{p}\log(T)\right\rceil$   & $\left\lceil\log\left(\frac{96(\sigma/\sqrt{\kappa}+\sigma_c)^2}{(\mu_*-\mu_{k})^2}\right)\right\rceil$ & $O\left(\sum_{k\not=k_*}\frac{\kappa(\sigma/\sqrt{\kappa}+\sigma_c)^2\log(T)}{(\mu_*-\mu_{k})}+C\frac{(\sigma/\sqrt{\kappa}+\sigma_c)^2\log(T)}{\Delta^2}\right)$   \\
		\bottomrule
	\end{tabular}\\
	\caption{Regret of Fed2-UCB algorithm with $f(p)=\kappa$ and different choices of $g(p)$. $\lambda$ and $\kappa$ are constants and $\Delta=\min_{k\not=k_*}\{\mu_*-\mu_k\}$ is the suboptimality gap; the $p_k$ column represents its upper bound.}
	\label{tbl:regret_Fed2-UCB}
\end{table}

\subsection{Discussion for Theorem~\ref{thm:regret_1}}\label{app:dis_regret1}
For Theorem \ref{thm:regret_1}, a few possible choices for $f(p)$ with the corresponding $p_k$ and asymptotic regrets are given in Table \ref{tbl:regret_Fed1-UCB}. While $f(p)=\kappa$, the overall asymptotic regret is independent of $M$; however, the communication loss is of order $O(\log(T))$. The choice of $f(p)=\left\lceil\kappa\log(T)\right\rceil$ results in a constant communication loss which scales as ${1}/{\Delta^2}$. When the update period grows exponentially, the communication loss is of order $O(\log(\log(T)))$ for $f(p)=2^p$ and $O(1)$ for $f(p)=\lceil2^p\log(T)\rceil$, and with these two choices, the communication loss now scales as $\log\left({1}/{\Delta}\right)$.

\begin{table}[htb]
	\centering
	\begin{tabular}{lll}
		\toprule
		$f(p)$     & $p_k$, $k\not=k_*$    & $R_1(T)$ \\
		\midrule
		\smallskip
		$\kappa$ & $\left\lceil\frac{96\sigma^2\log(T)}{\kappa M(\mu_*-\mu_k)^2}\right\rceil$  &  $O\left(\sum_{k\not=k_*}\frac{\sigma^2\log(T)}{(\mu_*-\mu_k)}+C\frac{\sigma^2\log(T)}{\kappa \Delta^2}\right)$   \\
		\smallskip
		$\left\lceil\kappa\log(T)\right\rceil$     &  $\left\lceil\frac{96\sigma^2}{\kappa M(\mu_*-\mu_k)^2}\right\rceil$ & $O\left(\sum_{k\not=k_*}\frac{\sigma^2\log(T)}{(\mu_*-\mu_k)}+C\frac{\sigma^2}{\kappa \Delta^2}\right)$      \\
		\smallskip
		$2^{p}$     &  $\left\lceil\log\left(\frac{96\sigma^2\log(T)}{ M(\mu_*-\mu_k)^2}\right)\right\rceil$      &  $O\left(\sum_{k\not=k_*}\frac{\sigma^2\log(T)}{(\mu_*-\mu_k)}+CM\log\left(\frac{\sigma^2\log(T)}{ M\Delta^2}\right)\right)$   \\
		$\left\lceil 2^{p}\log(T)\right\rceil$     & $\left\lceil\log\left(\frac{96\sigma^2}{ M(\mu_*-\mu_k)^2}\right)\right\rceil$       & $O\left(\sum_{k\not=k_*}\frac{\sigma^2\log(T)}{(\mu_*-\mu_k)}+CM\log\left(\frac{\sigma^2}{ M\Delta^2}\right)\right)$  \\
		\bottomrule
	\end{tabular}\\
		\caption{Regret of Fed1-UCB algorithm with different choices of $f(p)$. $\kappa$ is a constant and $\Delta=\min_{k\not=k_*}\{\mu_*-\mu_k\}$ is the suboptimality gap; the $p_k$ column represents its upper bound.}
	\label{tbl:regret_Fed1-UCB}
\end{table}

\section{Proofs for the Regret Analysis}
\subsection{Proof of Theorem~\ref{thm:accuracy}}
\begin{proof}
	With a union bound, we have
	\begin{equation}
	\begin{aligned}
		P_z &=P\left(\hat{\mu}^M_{k_*}\leq\max_{k\not=k_*}\hat{\mu}^M_k\right)= P\left(\bigcup_{k\not=k_*}\left(\hat{\mu}^M_{k_*}\leq\hat{\mu}^M_k\right)\right)\leq  \sum_{k\not=k_*}P\left(\hat{\mu}^M_{k_*}\leq\hat{\mu}_k^\congr{M}\right).
	\end{aligned}
	\label{eqn:appdx1}
	\end{equation}
    For a given arm $k\not=k_*$, we further have
	\begin{equation*}
	\begin{aligned}
	P\left(\hat{\mu}^M_{k_*}>\hat{\mu}^M_k\right)&\geq P\left(\hat{\mu}^M_{k_*}\geq \frac{1}{2}(\mu_k+\mu_*)\geq \hat{\mu}^M_k\right)\\
	& = P\left(\hat{\mu}^M_{k_*}\geq\frac{1}{2}(\mu_k+\mu_*)\right) P\left(\frac{1}{2}(\mu_k+\mu_*)\geq \hat{\mu}^M_k\right)\\
	&\overset{(i)}{\geq} \left(1-\exp\left\{\frac{-M(\mu_*-\mu_k)^2}{8\sigma_c^2}\right\}\right) \left(1-\exp\left\{\frac{-M(\mu_*-\mu_k)^2}{8\sigma_c^2}\right\}\right)\\
	&= 1-O\left(\exp\left\{\frac{-M(\mu_*-\mu_k)^2}{\sigma_c^2}\right\}\right).
	\end{aligned}
	\end{equation*}
	Inequality (i) is because $\hat{\mu}^M_{k_*}$ and $\hat{\mu}^M_k$ are $\frac{\sigma_c}{\sqrt{M}}$-subgaussian random variables. Thus, each term in the summation of Eqn.~\eqref{eqn:appdx1} can be bounded as
	\begin{equation*}
		P\left(\hat{\mu}^M_{k_*}\leq \hat{\mu}^M_k\right)= O\left(\exp\left\{-\frac{M(\mu_*-\mu_k)^2}{\sigma_c^2}\right\}\right).
	\end{equation*}
	Finally Theorem~\ref{thm:accuracy} can be derived as
	\begin{equation*}
			P_z \leq  \sum_{k\not=k_*}P\left(\hat{\mu}^M_{k_*}\leq \hat{\mu}^M_k\right)= O\left(\sum_{k\not=k_*}\exp\left\{-\frac{M(\mu_*-\mu_k)^2}{\sigma_c^2}\right\}\right)= O\left(K\exp\left\{-\frac{M\Delta^2}{\sigma_c^2}\right\}\right).
	\end{equation*}
\end{proof}

\subsection{Proof of Theorem~\ref{thm:regret_2}}
\subsubsection{Step 1: Confidence Bound for the Estimations}
We first analyze the probability guarantee of the interval for the averaged local mean estimation. In the Fed2-UCB algorithm, with two types of uncertainty, the following lemma provides an upper bound for the gap between averaged local means and the exact global means for each arm.
\begin{lemma}
	At phase $p$, for any active arm $k\in[K_p]$, it holds that
	\begin{equation*}
	P\left( \left|\bar{\mu}_k(p)-\mu_k \right|\geq B_{p,2}\right)\leq \frac{4}{T^3}.
	\end{equation*}
\end{lemma}
\begin{proof}
	The gap between $\bar{\mu}_k(p)$ and $\mu_k$ can be bounded as follows:
	\begin{equation*}
	\begin{aligned}
		&P\left(|\bar{\mu}_k(p)-\mu_k|\geq B_{p,2}\right)\\
		&= P\left(\left|\bar{\mu}_k(p)-\hat{\mu}^{M(p)}_k(p)+\hat{\mu}^{M(p)}_k(p)-\mu_k\right|\geq B_{p,2}\right)\\
		& \leq P\left(\left|\bar{\mu}_k(p)-\hat{\mu}^{M(p)}_k(p)\right|+\left|\hat{\mu}^{M(p)}_k(p)-\mu_k\right|\geq B_{p,2}\right)\\
		& = P\left(\left|\bar{\mu}_k(p)-\hat{\mu}^{M(p)}_k(p)\right|+\left|\hat{\mu}^{M(p)}_k(p)-\mu_k\right|\geq \sqrt{6\sigma^2\eta_p\log(T)}+\sqrt{\frac{6\sigma_c^2\log(T)}{M(p)}}\right)\\
		& \leq P\left( \left|\bar{\mu}_k(p)-\hat{\mu}^{M(p)}_k(p) \right|\geq \sqrt{6\sigma^2\eta_p\log(T)}\right)+P\left(\left|\hat{\mu}^{M(p)}_k(p)-\mu_k\right|\geq\sqrt{\frac{6\sigma_c^2\log(T)}{M(p)}}\right).
	\end{aligned}
	\end{equation*}
	
	For the first part, at phase $p$, the averaged local mean is $\bar{\mu}_k(p)=\frac{1}{M(p)}\sum_{m=1}^{M(p)}\bar{\mu}_{k,m}(p)$, while $\bar{\mu}_{k,m}(p)$ is the sample mean collected by client $m$. It can be observed that arm $k$ is pulled for $\sum_{q=1}^p f(q)=F(p)$ times by $g(1)$ clients (referred as ``group $1$''), $\sum_{q=2}^{p}=F(p)-F(1)$ times by $g(2)$ clients (``group $2$''), and so on until $f(p) = F(p)-F(p-1)$ times by $g(p)$ clients (``group $p$''). We also have that for clients in groups $1$, $\bar{\mu}_{k,m}$ is a $\frac{\sigma}{\sqrt{F(p)}}$-subgaussian random variable, while it is a $\frac{\sigma}{\sqrt{F(p)-F(1)}}$-subgaussian random variable for clients in group $2$, and so on. We further have that the overall average $\bar{\mu}_k(p)$ is a $\frac{\sigma}{M(p)}\sqrt{\sum_{q=1}^p\frac{g(q)}{F(p)-F(q-1)}}$-subgaussian random variable. With the sub-gaussian property and $\eta_p = \frac{1}{M(p)^2}\sum_{q=1}^p\frac{g(q)}{F(p)-F(q-1)}$, it holds that
	\begin{equation*}
		P\left(\left|\bar{\mu}_k(p)-\hat{\mu}^{M(p)}_k(p)\right|\geq \sqrt{6\sigma^2\eta_p\log(T)}\right)\leq 2\exp\left\{-\frac{6\sigma^2\eta_p\log(T)}{2\frac{\sigma^2}{M(p)^2}\sum_{q=1}^p\frac{g(q)}{F(p)-F(q-1)}}\right\}=\frac{2}{T^3}.
	\end{equation*}
	
	 For the second part, $\hat{\mu}^{M(p)}_k(p)=\frac{1}{M(p)}\sum_{m=1}^{M(p)}{\mu_{k,m}}$, which is a $\frac{\sigma_c}{\sqrt{M(p)}}$-subgaussian random variable. Thus, with the subgaussian property, we have
	\begin{equation*}
	\begin{aligned}
	&P\left(\left|\hat{\mu}^{M(p)}_k(p)-\mu_k\right|\geq\sqrt{\frac{6\sigma_c^2\log(T)}{M(p)}}\right)\leq 2\exp\left\{-\frac{\frac{6\sigma_c^2\log(T)}{M(p)}}{2\frac{\sigma_c^2}{M(p)}}\right\}= \frac{2}{T^3}.
	\end{aligned}
	\end{equation*}
	By combining the two parts together, the lemma is proved.
\end{proof}
Denote event $B=\left\{\forall p, \forall k\in[K_p],\left|\bar{\mu}_k(p)-\mu_k\right|\leq B_{p,2}\right\}$ and $P_b=\mathbb{P}(B)$. Since there are at most $T$ rounds and $K$ arms, with a simple union bound, we have
\begin{equation*}
P_b\geq 1-\frac{4K}{T^2}.
\end{equation*}

\subsubsection{Step 2: Required Number of Pulls}
Based on that event $B$ happens, the following lemma provides an upper bound for the required number of pulls to eliminate any sub-optimal arm.
\begin{lemma}\label{lem:round_elm_2}
	Assuming that event $B$ happens, for any sub-optimal arm $k\not=k_*$, there are at most $p_k$ rounds before arm $k$ is eliminated or the overall time runs out, where $p_k$ is the smallest integer satisfying
	\begin{equation*}
	96\left(\sigma\sqrt{\eta_p}+\sigma_c\frac{1}{\sqrt{M(p)}}\right)^2\log(T)\leq(\mu_*-\mu_k)^2.
	\end{equation*}
\end{lemma}
\begin{proof}
	Let $\Delta_k=\mu_*-\mu_k$ be the gap between the sub-optimal arm $k$ and the optimal arm and $p_k$ be the smallest integer such that $4B_{p_k,2}\leq \Delta_k$. Thus $p_k$ is the smallest integer satisfying
	\begin{equation*}
	96\left(\sigma\sqrt{\eta_p}+\sigma_c\frac{1}{\sqrt{M(p)}}\right)^2\log(T)\leq(\mu_*-\mu_k)^2.
	\end{equation*}
	For any $p\geq p_k$, we have
	\begin{equation*}
	\begin{aligned}
	&\Delta_k\geq 4B_{p,2};\\
	&|\bar{\mu}_k(p)-\mu_k|\leq B_{p,2};\\
	&|\bar{\mu}_{k_*}(p)-\mu_*|\leq B_{p,2}.
	\end{aligned}
	\end{equation*}
	With the above inequalities, we have
	\begin{equation*}
	\begin{aligned}
	\bar{\mu}_k(p)+B_{p,2}&\leq \mu_k+2B_{p,2}\leq \mu_k+2B_{p,2}+\bar{\mu}_{k_*}(p)-\mu_*+B_{p,2}\\
	&=-(\Delta_k-4B_{p,2})+\bar{\mu}_{k_*}(p)-B_{p,2}\leq \bar{\mu}_{k_*}(p)-B_{p,2},
	\end{aligned}
	\end{equation*}
	which means arm $k$ is eliminated. Thus, arm $k$ is pulled at most $p_k$ rounds before elimination or the overall time runs out.
\end{proof}

\subsubsection{Step 3: Overall Regret}
When event $B$ holds, the overall regret, denoted as $R_{s,2}(T)$, can be decomposed as
\begin{equation*}
R_{s,2}(T)=\sum_{k=1}^K(\mu_*-\mu_k)\mathbb{E}[N(k)]+C\mathbb{E}\left[\sum_{\tau=1}^{T_c}M_{\gamma_\tau}\right],
\end{equation*}
where $N(k)$ is the overall number of pulls on arm $k$. For the first term, i.e. the exploration loss,  for any sub-optimal arm $k$, with $p_k$ defined in Lemma~\ref{lem:round_elm_2}, we have
\begin{equation*}
(\mu_*-\mu_k)\mathbb{E}[N(k)]\leq (\mu_*-\mu_k)\sum_{p=1}^{p_k}M(p)f(p).
\end{equation*}
The communication loss is similarly determined by $p_{\max}=\max_{k\not=k_*}\{p_k\}$, which satisfies
\begin{equation*}
C\sum_{\tau=1}^{T_c}M_{\gamma_\tau} \leq C\sum_{q=1}^{p_{\max}}M(q).
\end{equation*}
Then $R_{s,2}(T)$ can be bounded as
\begin{equation*}
R_{s,2}(T) \leq \sum_{k\not=k_*}(\mu_*-\mu_k)\sum_{p=1}^{p_k}M(p)f(p)+C\sum_{q=1}^{p_{\max}}M(q).
\end{equation*}

If event $B$ does not hold, with $\beta T$ as an upper bound for the number of clients, the exploration regret and communication regret are upper bounded by a linear loss $\beta T^2$ and $\beta CT^2$ respectively. Thus, the regret of this case, denoted as $R_{f,2}(T)$, can be bounded as
\begin{equation*}
R_{f,2}(T) \leq \beta (1+C)T^2.
\end{equation*}
With $R_{s,2}(T)$ and $R_{f,2}(T)$, the overall regret $R_2(T)$ can be finally bounded as
\begin{align*}
R_2(T)&=P_bR_{s,2}(T)+(1-P_b)R_{f,2}(T)\\
&\leq R_{s,2}(T)+(1-P_b)R_{f,2}(T)\\
&\leq \sum_{k\not=k_*}(\mu_*-\mu_k)\sum_{p=1}^{p_k}M(p)f(p)+C\sum_{q=1}^{p_{\max}}M(q)+4\beta(1+C).
\end{align*}

\subsection{Proof of Theorem~\ref{thm:regret_1}}
\congr{The proof of Theorem~\ref{thm:regret_1} follows similar steps as the proof of Theorem~\ref{thm:regret_2}.}
\subsubsection{Step 1: Confidence Bound for the Estimations}
\begin{lemma}
	At phase $p$, for any active arm $k\in[K_p]$, it holds that
	\begin{equation*}
		P\left(|\bar{\mu}_k(p)-\mu_k|\geq B_{p,1}\right)\leq \frac{2}{T^2}.
	\end{equation*}
\end{lemma}
\begin{proof}
	Since $\bar{\mu}_k(p)=\frac{1}{M}\sum_{m=1}^M\bar{\mu}_{k,m}(p)$ while $\bar{\mu}_{k,m}(p)$ is the sample mean collected by client $m$ through $F(p)$ pulls, which thus is a $\frac{\sigma}{\sqrt{F(p)}}$-subgaussian random variable, $\bar{\mu}_k(p)$ is a $\frac{\sigma}{\sqrt{MF(p)}}$-subgaussian random variable. Thus, with the subgaussian property, we have
	\begin{equation*}
	\begin{aligned}
		&P\left(\left|\bar{\mu}_k(p)-\mu_k\right|\geq B_{p,1}\right)\leq 2\exp\left\{-\frac{MF(p)B_{p,1}^2}{2\sigma^2}\right\}= 2\exp\left\{-\frac{MF(p)\frac{6\sigma^2\log(T)}{MF(p)}}{2\sigma^2}\right\}= \frac{2}{T^3}.
	\end{aligned}
	\end{equation*}
\end{proof}
Denoting event $A=\left\{\forall p, \forall k\in[K_p],|\bar{\mu}_k(p)-\mu_k|\leq B_{p,1}\right\}$ and $P_a=\mathbb{P}(A)$, since there are at most $T$ rounds and $K$ arms, with a simple union bound, we have
\begin{equation*}
P_a\geq 1-\frac{2K}{T^2}.
\end{equation*}

\subsubsection{Step 2: Required Number of Pulls}
\begin{lemma}\label{lem:round_elm_1}
	Assuming that event $A$ happens, for any sub-optimal arm $k\not=k_*$, there are at most $p_k$ rounds before it is eliminated or the overall time runs out, where $p_k$ is the \congr{smallest} integer satisfying
	$$
	MF(p_k)\geq  \frac{96\sigma^2\log(T)}{(\mu_*-\mu_k)^2}.
	$$
\end{lemma}
\begin{proof}
	Let $\Delta_k=\mu_*-\mu_k$ be the gap between the sub-optimal arm $k$ and the optimal arm, and $p_k$ be the smallest integer such that $4B_{p_k,1}\leq \Delta_k$. We have
	\begin{equation*}
	    MF(p_k)\geq  \frac{96\sigma^2\log(T)}{(\mu_*-\mu_k)^2}.
	\end{equation*}
	If $p_k\leq F^{-1}(T)$, for such $p\geq p_k$, it leads to
	\begin{equation*}
	\begin{aligned}
		&\Delta_k\geq 4B_{p,1};\\
		&|\bar{\mu}_k(p)-\mu_k|\leq B_{p,1};\\
		&|\bar{\mu}_{k_*}(p)-\mu_*|\leq B_{p,1}.
	\end{aligned}
	\end{equation*}
	With the above inequalities, we can further derive
	\begin{equation*}
	\begin{aligned}
	\bar{\mu}_k(p)+B_{p,1}&\leq \mu_k+2B_{p,1}\leq \mu_k+2B_{p,1}+\bar{\mu}_{k_*}(p)-\mu_*+B_{p,1}\\
	&=-(\Delta_k-4B_{p,1})+\bar{\mu}_{k_*}(p)-B_{p,1}\leq \bar{\mu}_{k_*}(p)-B_{p,1},
	\end{aligned}
	\end{equation*}
	which means arm $k$ is eliminated. Thus, arm $k$ is pulled for at most $p_k$ phases by each client before elimination or the overall time runs out.
\end{proof}

\subsubsection{Step 3: Overall Regret}
When event $A$ holds, the overall regret, denoted as $R_{s,1}(T)$, can be decomposed as
\begin{equation*}
	R_{s,1}(T)=\sum_{k=1}^K(\mu_*-\mu_k)\mathbb{E}[N(k)]+\mathbb{E}\left[CMT_c\right],
\end{equation*}
where $N(k)$ is the overall number of pulls on arm $k$ by all the clients. For the first term, i.e. the exploration loss,  with Lemma~\ref{lem:round_elm_1}, for any sub-optimal arm $k$, \congr{if it holds that}
\begin{equation*}
    \frac{96\sigma^2\log(T)}{(\mu_*-\mu_k)^2}\leq MF(p_k)
\end{equation*}
\congr{at round $p_k$,} we can conclude \congr{arm $k$} is eliminated in this round. Thus,
\begin{equation*}
\begin{aligned}
(\mu_*-\mu_k)\mathbb{E}[N(k)]&\leq M(\mu_*-\mu_k)F(p_k).
\end{aligned}
\end{equation*}

Since there are no more communications after the optimal arm is found or the overall time runs out, the communication loss is determined by $p_{\max}=\max_{k\not=k_*}\{p_k\}$ and can be bounded as
\begin{equation*}
	CMT_c \leq CM\min\left\{\max_{k\not=k_*}\{p_k\},F^{-1}(T)\right\} \leq CMp_{\max}.
\end{equation*}
Then $R_{s,1}(T)$ can be bounded as
\begin{equation*}
    R_{s,1}(T) \leq M\sum_{k\not=k_*}(\mu_*-\mu_k)F(p_k)+CMp_{\max}.
\end{equation*}

If event $A$ does not hold, the exploration regret can be simply upper bounded by a linear loss $MT$ while the communication loss is also simply upper bounded linearly by $CMT$. The regret of this case, denoted as $R_{f,1}(T)$, can be bounded as
\begin{equation*}
	R_{f,1}(T) \leq (1+C)MT.
\end{equation*}

With $R_{s,1}(T)$ and $R_{f,1}(T)$, the overall regret $R_1(T)$ can finally be bounded as
\begin{align*}
	R_1(T)&=P_aR_{s,1}(T)+(1-P_a)R_{f,1}(T)\\
	&\leq R_{s,1}(T)+(1-P_a)R_{f,1}(T)\\
	&\leq M\sum_{k\not=k_*}(\mu_*-\mu_k)F(p_k)+CMp_{\max}+2(1+C)MK/T.
\end{align*}

\section{Algorithm Enhancements}
\subsection{Better Privacy Protection}
The clients in Fed1-UCB and Fed2-UCB algorithms are assumed to share complete information about their sample means. Although sharing sample means instead of raw samples from pulling arms already provides some privacy protection, we propose two enhanced schemes to further improve privacy protection while keeping the overall regret performance of the same order.

First, clients can communicate a quantized sample mean while \congr{maintaining correct} arm eliminations at the server. The key idea is that \textit{if the quantization error does not dominate the confidence bound, such errors can be tolerated.} For example, for the Fed1-UCB algorithm, if $\forall k\in[K_p], \forall m\in[M],\mu_{k,m}\in[0,1]$, at phase $p$, the clients can share a quantized sample mean of length $Q_1 =1+\log_2\left(\sqrt{\frac{6\sigma^2\log(T)}{MF(p)}}\right)$, while the server reconstructs the principle of rejection by replacing $B_{p,1}$ with $B_{p,1}'=B_{p,1}+\sqrt{\frac{6\sigma^2\log(T)}{MF(p)}}$. Similarly, for the Fed2-UCB algorithm, the clients can use $Q_2=1+\log_2\left(\sqrt{\frac{6\sigma_c^2\log(T)}{M(p)}}\right)$ while the central server can reform $B_{p,2}$ as $B_{p,2}'=B_{p,2}+\sqrt{\frac{6\sigma_c^2\log(T)}{M(p)}}$. Theoretical analysis shows that such modifications lead to a slightly larger but still constant multiplicative factor, but do not change the overall order of the regret. Also, sharing quantized sample means can reduce the communication cost significantly when massive clients are participating in the FMAB task \cite{konevcny2016federated}.

Second, when there are a large number of clients participating in the FMAB game, each individual client can add a zero-mean random noise to their transmitted sample means, e.g. a zero-mean Gaussian or Laplace noise. In this way, the server can only access a noisy sample mean of each individual client, which protects the sensitive data. At the same time, since the collected sample means from clients are averaged at the server, the arm eliminations remain successful based on the averaged sample means with a high probability due to the concentration inequality.

\section{Experiment Details and Additional Results}
In this section, we provide details of the experimentation that due to space limit were not included in the main paper. All the experimental results presented here and in the main paper are obtained by averaging over $100$ experiments, with error bars reported when applicable.
\begin{figure*}[htb]
    \centering
	\begin{minipage}[t]{0.35\linewidth}
		\centering
		\includegraphics[width=\linewidth]{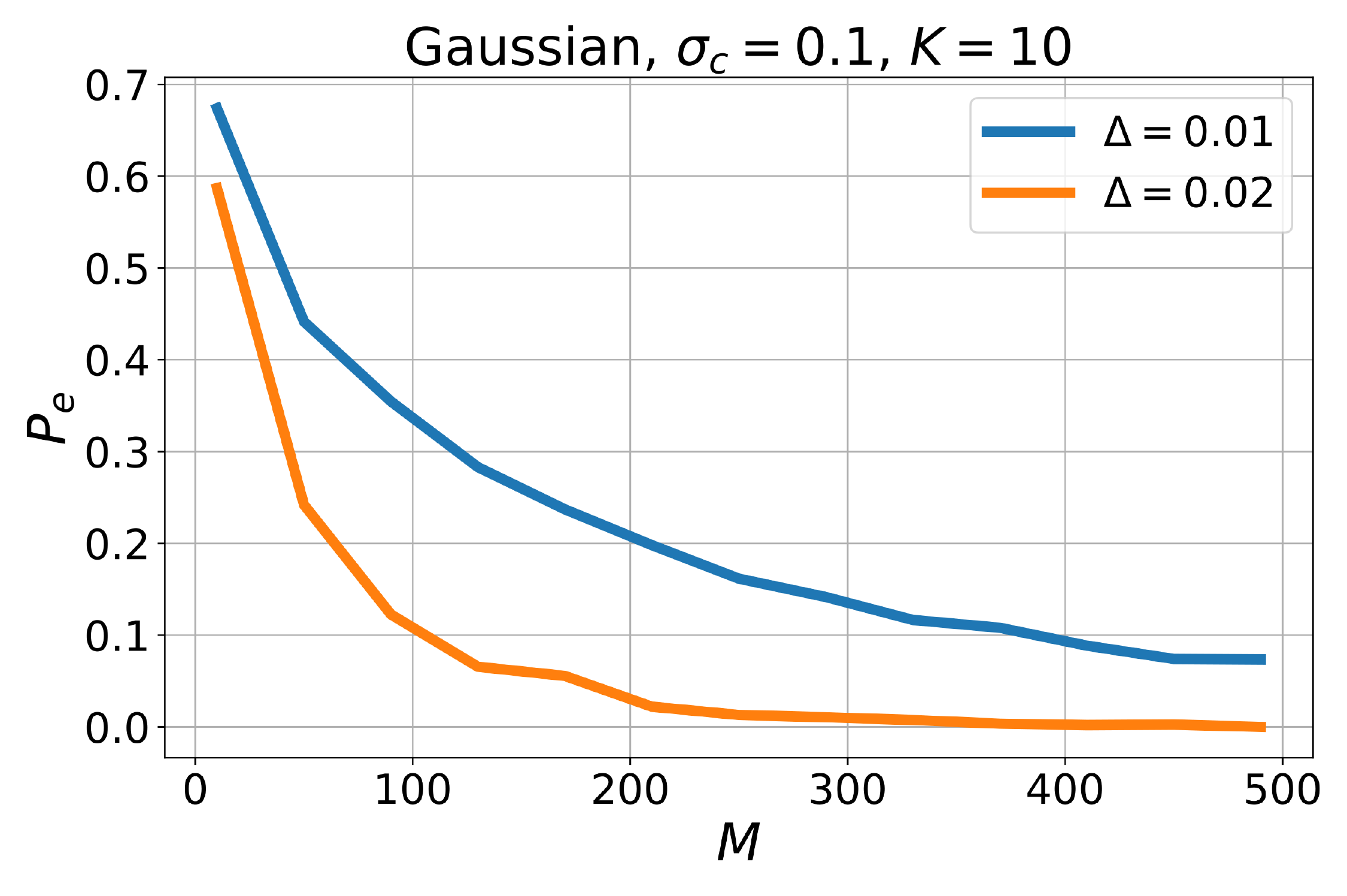}
		\caption{$P_e$ with different $M$.}
		\label{fig:acc_gaussian}
	\end{minipage}
		\begin{minipage}[t]{0.35\linewidth}
    	\centering
    	\includegraphics[width=\linewidth]{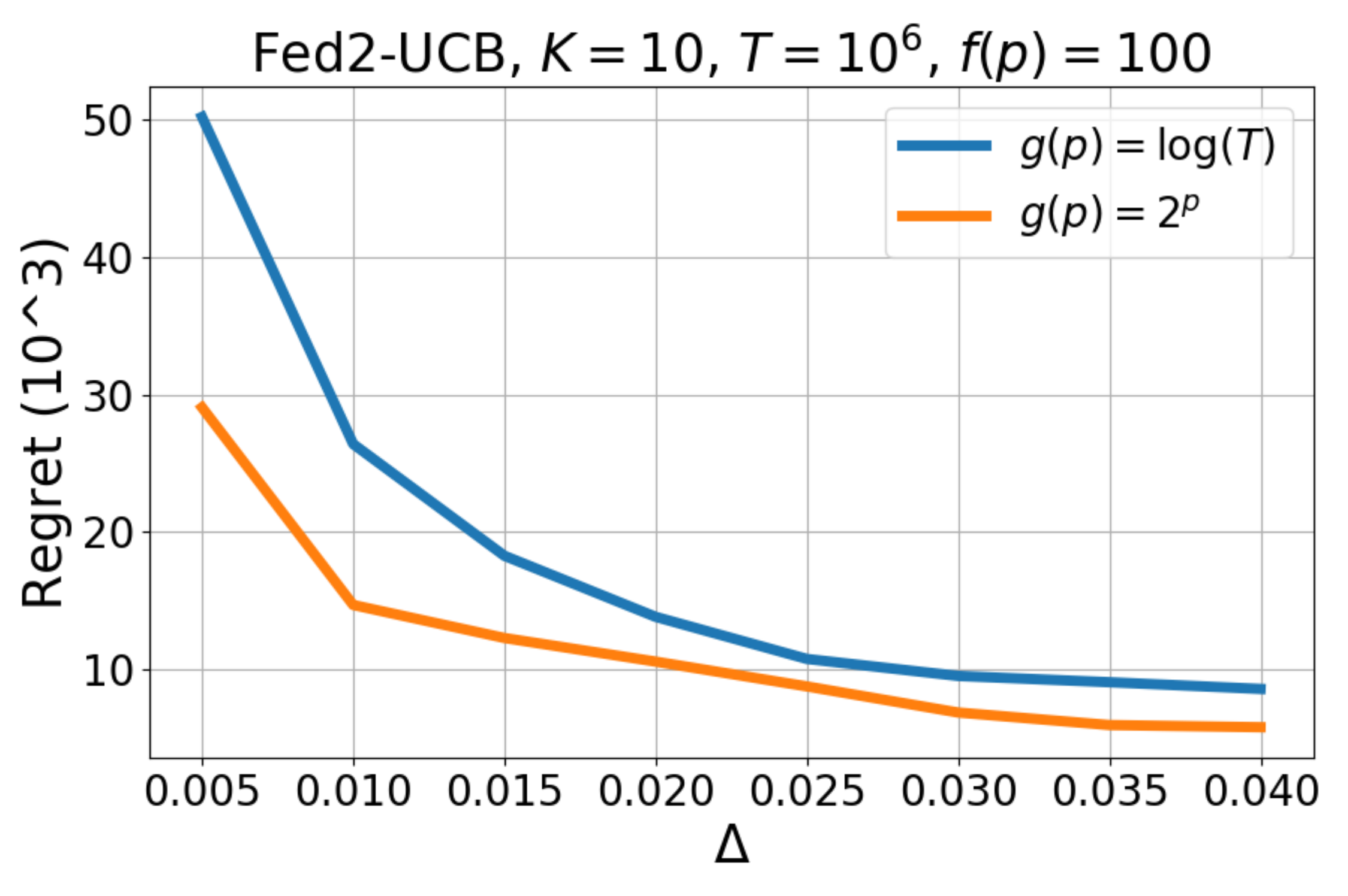}
    	\caption{Fed2-UCB with different $\Delta$.}
    	\label{fig:fed2_ucb_delta}
	\end{minipage}
\end{figure*}

\subsection{Code and Dataset}
The implementation codes of the Fed2-UCB and Fed1-UCB algorithm that we used for simulations have been \congr{made publically available at \url{https://github.com/ShenGroup/FMAB}}, which also contains the synthetic dataset and the pre-processed real-world MovieLens dataset. 
We note that the original version of the MovieLens dataset is publically avaialble at \url{https://grouplens.org/datasets/hetrec-2011/}.

\subsection{Additional Results}
First, the relationship between $P_e$ and $M$ is verified with the synthetic dataset as shown in Fig.~\ref{fig:acc_gaussian}. Using a zero-mean Gaussian distribution with $\sigma_c=0.1$ to generate local models, we can clearly see that $P_e$ exponentially decreases with $M$ with both $\Delta=0.02$ and $\Delta=0.01$. Also, with $\Delta=0.02$, $P_e$ is smaller than the case of $\Delta=0.01$, which corroborates Theorem \ref{thm:accuracy}. 

Then, we focus on the performance of Fed2-UCB with the choice $g(p)=\left\lceil\log(T)\right\rceil$ and $g(p)=2^p$, and report the results in Fig.~\ref{fig:fed2_ucb_delta}. We can see that with $g(p)=2^p$, Fed2-UCB achieves a better regret under various bandit environments. Also, the choice $g(p)=\left\lceil\log(T)\right\rceil$ leads to a stronger dependence on $\Delta$ than $g(p)=2^p$, which coincides with the theoretical analysis in Table \ref{tbl:regret_Fed2-UCB}. 
\end{document}